\documentclass{article}

\usepackage{microtype}
\usepackage{graphicx}
\usepackage{subfigure}
\usepackage{booktabs} % for professional tables

\usepackage{hyperref}

\usepackage[accepted]{icml2024}

\usepackage{amsmath}
\usepackage{amssymb}
\usepackage{mathtools}
\usepackage{amsthm}

\usepackage[capitalize,noabbrev]{cleveref}

\theoremstyle{plain}
\newtheorem{theorem}{Theorem}[section]

\newtheorem{lemma}[theorem]{Lemma}
\newtheorem{corollary}[theorem]{Corollary}
\newtheorem{condition}[theorem]{Condition}
\theoremstyle{definition}

\theoremstyle{remark}

\usepackage{nccmath}
\usepackage{comment}
\usepackage{tikz}
\usetikzlibrary{decorations.pathreplacing}
\definecolor{goldenpoppy}{rgb}{0.99, 0.76, 0.0}
\definecolor{glaucous}{rgb}{0.38, 0.51, 0.71}
\definecolor{blue-violet}{rgb}{0.54, 0.17, 0.89}

\usepackage[textsize=tiny]{todonotes}

\icmltitlerunning{Locally Interdependent Multi-Agent MDP}

\begin{document}

\twocolumn[
%\icmltitle{Submission and Formatting Instructions for \\
%           International Conference on Machine Learning (ICML 2024)}

\icmltitle{Locally Interdependent Multi-Agent MDP: Theoretical Framework for Decentralized Agents with Dynamic Dependencies}

% It is OKAY to include author information, even for blind
% submissions: the style file will automatically remove it for you
% unless you've provided the [accepted] option to the icml2024
% package.

% List of affiliations: The first argument should be a (short)
% identifier you will use later to specify author affiliations
% Academic affiliations should list Department, University, City, Region, Country
% Industry affiliations should list Company, City, Region, Country

% You can specify symbols, otherwise they are numbered in order.
% Ideally, you should not use this facility. Affiliations will be numbered
% in order of appearance and this is the preferred way.
\icmlsetsymbol{equal}{*}

\begin{icmlauthorlist}
\icmlauthor{Alex DeWeese}{carnegie}
\icmlauthor{Guannan Qu}{carnegie}
\end{icmlauthorlist}

\icmlaffiliation{carnegie}{Department of Electrical and Computer Engineering, Carnegie Mellon University, Pittsburgh PA, USA}

\icmlcorrespondingauthor{Alex DeWeese}{mdeweese@andrew.cmu.edu}
\icmlcorrespondingauthor{Guannan Qu}{gqu@andrew.cmu.edu}

% You may provide any keywords that you
% find helpful for describing your paper; these are used to populate
% the "keywords" metadata in the PDF but will not be shown in the document
\icmlkeywords{Machine Learning, ICML}

\vskip 0.3in
]

% this must go after the closing bracket ] following \twocolumn[ ...

% This command actually creates the footnote in the first column
% listing the affiliations and the copyright notice.
% The command takes one argument, which is text to display at the start of the footnote.
% The \icmlEqualContribution command is standard text for equal contribution.
% Remove it (just {}) if you do not need this facility.

\printAffiliationsAndNotice{}  % leave blank if no need to mention equal contribution
%\printAffiliationsAndNotice{\icmlEqualContribution} % otherwise use the standard text.

\begin{abstract}
Many multi-agent systems in practice are decentralized and have dynamically varying dependencies. There has been a lack of attempts in the literature to analyze these systems theoretically. In this paper, we propose and theoretically analyze a decentralized model with dynamically varying dependencies called the Locally Interdependent Multi-Agent MDP. This model can represent problems in many disparate domains such as cooperative navigation, obstacle avoidance, and formation control. Despite the intractability that general partially observable multi-agent systems suffer from, we propose three closed-form policies that are theoretically near-optimal in this setting and can be scalable to compute and store. 
Consequentially, we reveal a fundamental property of Locally Interdependent Multi-Agent MDP's that the partially observable decentralized solution is exponentially close to the fully observable solution with respect to the visibility radius. We then discuss extensions of our closed-form policies to further improve tractability. We conclude by providing simulations to investigate some long horizon behaviors of our closed-form policies.

\end{abstract}

\section{Introduction}
Many important real-world multi-agent applications are decentralized and have local dynamic relationships between agents. Examples cut across a broad spectrum of areas including obstacle avoidance, cooperative navigation, formation control. The decentralized nature of the problem presents a key challenge because agents have limited access to the state information of other agents. This may be introduced either artificially to improve scalability or by physical constraints caused by communication limitations. Another challenge is, that while each agent has their own dynamics, they are interacting with the set of nearby agents that are dynamically varying. The presence of such dynamic interactions makes the decision problem for different agents highly coupled and nontrivial.

With the recent successes in Reinforcement Learning (RL), there has been an increasing desire to use this emerging tool for these types of decentralized multi-agent systems with locally dynamic agent interactions. However, most of the work in this area is empirical (\cite{lowe2017multi, han2020cooperative, batra2022decentralized, long2018towards, baldazo2019decentralized, zhou2019learn, palanisamy2020multi, aradi2020survey}). This is because it is non-trivial to create a model that accurately represents the decentralized partial observability and dynamic local dependencies of agents while producing theoretically sound and relatively tractable solutions.
In addition to the empirical work, there has been a growing literature on theoretical MARL, but none attempt to model and solve the challenges brought by the decentralized partial observability and local dynamic agent interactions \cite{qu2022scalable, qu2020scalableaverage, qu2020scalablepolicy, lin2021multi, kar2013cal,zhang2018fully, suttle2020multi, chen2021communication}.

The discussions above naturally motivate the question: 

\emph{Can we theoretically model and find near-optimal policies for decentralized agents with dynamic local dependencies in a scalable manner? }

We answer this question affirmatively with our contributions described below. 

\subsection{Contributions}
\label{contributions}
For our first contribution, we propose a novel setting called the Locally Interdependent Multi-Agent MDP. It is a cooperative setting that consists of multiple agents acting in a common metric space. 
It models dynamic proximity-based relationships that allow agents within a distance $\mathcal R$ to influence and depend on one another. Furthermore, agents within proximity $\mathcal V$ are also permitted to communicate with each other. As the location of the agents can dynamically change, both the ``dependence graph'' and the ``communication graph'' can be dynamic and time-varying. 

As a second contribution, we provide near-optimal solutions and establish a fundamental property of this setting. Specifically, we attempt to answer the following two questions: 

\textit{1) Are there decentralized 
policies in this setting that perform well theoretically?} 

We provide three closed-form policy solutions we call the Amalgam Policy, the Cutoff Policy, and the First Step Finite Horizon Optimal Policy. The policies in our decentralized framework have nearly the best possible theoretical performance guarantees in this setting. This is shown by providing an upper bound for the performance of each of policy we construct (\cref{amal_bound}, \cref{cutoff_bound}, \cref{finite_bound}) and matching them with a performance lower bound we also construct (\cref{lower_bound}), up to constant factors. 

\textit{2) How significantly does decentralized partial observability impact the best possible performance?} 

As a corollary, we reveal that the performance of the optimal policy in our partially observable decentralized framework will approximate the optimal fully observable centralized policy exponentially well with the increase in the visibility radius. This establishes a fundamental property of our decentralized framework in Locally Interdependent Multi-Agent MDP's.

Lastly, we give insight into the practical uses and behavior of our policies. We present some extensions of our solutions that can further improve scalability, and provide simulations of toy problems in obstacle avoidance, cooperative navigation, and formation control. They demonstrate the long term behaviors of the policies and act as a proof of concept for applications in more complex systems.

Together, these contributions constitute a novel framework that has wide applicability and has closed-form solutions with rigorous theoretical guarantees that are implementable in practice.

\subsection{Related Work}
Our work is related to two bodies of literature.

Firstly, there exist many lines of empirical work that consider a multi-agent setting with decentralized agents and dynamic local dependencies. In fact, many resemblances of our model show up in practical applications in problems such as cooperative navigation, obstacle avoidance, and formation control \cite{han2020cooperative, batra2022decentralized, lowe2017multi}. Some examples are robot navigation \cite{han2020cooperative, long2018towards}, autonomous driving \cite{palanisamy2020multi, aradi2020survey}, UAV's \cite{baldazo2019decentralized, batra2022decentralized}, and USV's \cite{zhou2019learn}. However, all of these works are empirical in nature and our goal is to provide a theoretical framework for these types of problems.

Secondly, our work is also related to the large body of the theoretical MARL literature. This is a very broad area and some representative works include  V-learning \cite{bai2020near, jin2021v, wang2023breaking}, mean-field RL \cite{yang2018mean, gu2021mean, carmona2019model}, and function approximation \cite{xie2020learning, jin2022power, huang2021towards}. In this literature, the areas most related to us are 1) the Dec-POMDP, 2) scalable RL, and 3) decentralized RL. 

1) Our setting is indeed a special case of the Dec-POMDP \cite{oliehoek2012decentralized, oliehoek2016concise}. Unfortunately, solving a general Dec-POMDP is NEXP-Complete. Similar to our approach, there have been many attempts to consider special cases of the Dec-POMDP to improve tractability. However, as far as we know, these methods inherit the hardness of the Dec-POMDP \cite{goldman2004decentralized, bernstein2002complexity, allen2009complexity}, cannot model the setting we consider, or are not theoretically analyzed. A close example to our setting in the partially observable multi-agent literature is the Interaction Driven Markov Game (IDMG) \cite{spaan2008interaction, melo2009learning} which has many related aspects to our setting such as rewards decomposing into a single-agent component and a multi-agent component, but the theory is not well explored. Compared to the IDMG, we make some additional assumptions, but we believe the key insight that improves the theoretical quality of our solution is to allow coordination between agents for several time steps prior to reward dependence (see the \cref{dtl}). Another example of a related setting is the Network Distributed POMDP \cite{nair2005networked, zhang2011coordinated} which has some resemblance of our setting such as transition independence and group reward dependence. However, it assumes a fixed reward grouping and independent observations, neither of which holds in our setting. As far as we know, the literature in this area is not able to solve problems in our setting in a theoretically sound, tractable, and scalable manner. We take a unique approach to circumvent the difficulties of analyzing general partial observability structures. Namely, we propose reasonable closed-form solutions for a particular environment and observability structure and then theoretically prove their effectiveness.

2) The scalable RL literature considers a setting very close to ours \cite{qu2022scalable, qu2020scalableaverage, qu2020scalablepolicy, lin2021multi}. It considers decentralized agents with local dependencies much like our setting, however, the dependence network between agents is fixed or stochastic and the dependence is in terms of the probability transitions. Our work has dynamic time-varying dependence networks between agents and the dependence is in terms of the rewards. Nevertheless, it can be considered an extension of works in this area and resemblances of relevant concepts such as the exponential decay property can be seen in this work. 

3) The distributed RL literature considers multiple agents connected through a possibly stochastic network \cite{ kar2013cal,zhang2018fully, suttle2020multi, chen2021communication}. It differs in that the objective of the distributed RL literature is to find the joint optimal policy for all the agents by exchanging reward information with its neighbors. In our case, agents will exchange state information with other agents and act only according to that acquired state information.

\section{Preliminaries}

\subsection{Locally Interdependent Multi-Agent MDP}
We consider a setting with a set of agents $\mathcal{N} = \{1,..., n\}$. Each agent $k$ is associated with a state $x_k$ from a common metric space $\mathcal X$ with an associated distance metric $d$. This state $x_k$ is the ``location'' of agent $k$. Further, each agent $k\in \mathcal N$ will also have an internal state from a local state space $\mathcal Y_k$. Thus the state for agent $k$ can be represented as $s_k = (x_k, y_k) \in \mathcal S_k$ where $x_k \in \mathcal X$ and $y_k \in \mathcal Y_k$.

Given the distance metric $d$, agents $j, k \in \mathcal N$ at states $s_j = (x_j, y_j) \in \mathcal S_j$, $s_k = (x_k, y_k) \in \mathcal S_k$ are then associated with a notion of distance. We overload notation and denote the distance $d$ between states as $d(s_j, s_k) = d(x_j, x_k)$.

Each agent takes an action $a_k$ that lies in a local action space $\mathcal A_k$. Given $s_k(t),a_k(t)$ at time $t$, agent $k$'s state transitions according to a local transition function $P_k$, i.e. $s_k(t+1)\sim P_k(\cdot|s_k(t),a_k(t))$. We assume agents will not travel more than a distance of $1$ at each step. That is, $P_k(s'_k \lvert s_k,a_k) = 0$ if $d(s_k, s'_k) > 1$. Since the distance of 1 is a unitless distance dictated by the metric space, this simply reflects the assumption that movement distance at each step is bounded.

Agents will cooperate and take actions to accumulate discounted rewards that have independent and interdependent components. Agents will obtain independent local rewards according to a local reward function $\overline{r}_j(s_j, a_j)$ at each step. Furthermore, for other agents $k$ within a distance of a dependence constant $\mathcal R$ of agent $j$, an arbitrary reward of $\overline r_{j,k}(s_j, a_j, s_k, a_k)$ will be added. These rewards will then be discounted by factor $\gamma$ throughout the time steps.

This proximity-based reward function incentivizes each individual agent while providing rewards and penalties for agents within its vicinity according to the distance metric $d$. As agents move throughout the space, they will begin to influence each other in a time-varying dynamic way. 
If we assert $\overline r_{j,k}(s_j, a_j, s_k, a_k) = 0$ when $d(s_j, s_k) > \mathcal R$, we can compactly say $r_j(s, a) = \sum_{k\in \mathcal N} \overline r_{j,k}(s_j, a_j, s_k, a_k)$ where if $j,k$ are equal, we say $\overline r_{j,j} (s_j, a_j, s_j, a_j) = \overline r_j(s_j, a_j)$.
\\\\\\\\
To summarize, we define the Locally Interdependent Multi-Agent MDP to be $\mathcal{M} = (\mathcal{S}, \mathcal A, P, r, \mathcal R, \gamma)$:
\begin{itemize}
\item $\mathcal S := \mathcal S_1 \times ... \times \mathcal S_n$ 
\item $\mathcal A := \mathcal A_1 \times ... \times \mathcal A_n$ 
\item $ P(s' \lvert s, a) = \prod_{k \in \mathcal N} P_k(s'_k \lvert s_k,a_k)$ where $P_k(s'_k \lvert s_k,a_k) = 0$ if $d(s_k, s'_k) > 1$
\item $ r(s, a) = \sum_{j, k\in \mathcal N} \overline r_{j,k}(s_j,a_j, s_k,a_k)$ where $\overline r_{j,k}(s_k, a_k, s_j, a_j) = 0$ when $d(s_j, s_k) > \mathcal R$. $j,k$ may be equal.
\end{itemize}

For the subsequent sections, we denote $\tilde r = \sup_{s, a} \lvert r(s,a)\rvert$

We will also indicate a finite realizable trajectory in this MDP to be a sequence of state actions $(s(t), a(t))$ for $t \in \{0,  ..., T\}$ that have the property $P(s(t + 1)\lvert s(t), a(t)) > 0$ for all $t \in \{0, ..., T - 1\}$. This is similarly defined for infinite realizable trajectories.

\subsection{Group Decentralized Policies}
\label{group_decentralized}
To introduce partial observability, we say two agents can communicate with each other and take a coordinated action if there is a path of adjacent agents each within a constant distance $\mathcal V$ of each other with $\mathcal V> \mathcal R$. Formally, for $s \in \mathcal S$, the visibility partition on the agents $\mathcal N$ denoted as $Z(s)$ is defined by the equivalence relation such that for agents $j,k \in \mathcal N$, there is a sequence of agents $n_0, ..., n_{\ell}$ such that  $n_0 = j, n_\ell = k$ and $d(n_{i}, n_{i + 1}) \leq \mathcal V$ for $i \in\{0,..., \ell-1\}$. We say two agents within the same visibility partition can then communicate with one another (see \cref{imamdp_visual}). 

Therefore our partially observable policy class of interest $\Pi_{\mathcal V}$ contain policies $\pi: \mathcal S \rightarrow \Delta(\mathcal A)$ that take the form $\pi(s) = (\pi_z(s_z): z \in Z(s))$ where agents within the same visibility group $z\in Z(s)$ act according to some coordinated policy $\pi_z$ in $\mathcal M_z$ 
 which is the associated Locally Interdependent Multi-Agent MDP defined on a subset of agents $z\subset \mathcal N$ with the same $S_j$, $A_j$, $P_j$, $\overline r_{j,k}$ for $j,k \in z$. In other words, for $z\in Z(s)$,  $\pi_z$ will be a centralized policy based solely on the states of agents in the communication group and can thus be executed by our communication assumptions. We call this the group decentralized policy class.
 
We will generalize this notation and define a concatenation of policies. For a partition $P(s)$ on $\mathcal N$ possibly depending on $s$, and policies $\pi_p(s_p)$ for $p \in P(s)$, the concatenation of policies is denoted as $\pi(s) = (\pi_p(s_p) : p \in P(s))$. 

Throughout this paper, we also indicate a sample trajectory $\tau = (s(t), a(t))$ distributed from a policy $\pi$ with starting state $s$ and starting action as $a$ as $\tau \sim \pi\lvert_{s, a}$. If only the starting state is conditioned, we use $\tau \sim \pi \lvert_s$.

\subsection{Scalability}
\label{decentralized_scalable}
Notice that the size of this group decentralized policy class can be significantly smaller than all possible policies. To illustrate, we first establish a Locally Interdependent Multi-Agent MDP with $kM$ agents, a finite $\lvert \mathcal X \rvert$, finite $\lvert \mathcal A_1\rvert$ with $\lvert \mathcal A_1\rvert = \lvert \mathcal A_i\rvert$ for all $i$, and a trivial internal state space for each agent. Now consider the policy class with policies of the form $\pi(s) = (\pi_g(s_g): g \in G)$ for some fixed partition $G$ with $M$ agents in each group. The space required to store a group decentralized policy will be $k\lvert \mathcal X \rvert ^M \lvert \mathcal A_1\rvert^M$ elements as opposed to $\lvert \mathcal X\rvert ^{kM}\lvert \mathcal A_1\rvert^{kM}$ elements, which is a very dramatic difference numerically. Furthermore, the space required to store a deterministic group decentralized policy is reduced from $\lvert \mathcal X\rvert ^{kM}$ to $k\lvert \mathcal X \rvert ^M $ elements.

In the group decentralized policy class ``$G$'' changes dynamically with the state so the numerical advantages are MDP dependent. However, in practice, the improvements can be quite substantial (see \cref{bullseye_many} for an example).
 The size reduction is reflected in the reduced space required to store these policies as well as the reduced computation time of some algorithms used to find them. We demonstrate this with the algorithms used to find the Cutoff and First Step Finite Horizon Optimal Policies, to be introduced later in \cref{sec:main_results} and extensions in \cref{extensions}.

\begin{figure}[h]
\centering
\includegraphics[width=0.3\textwidth]{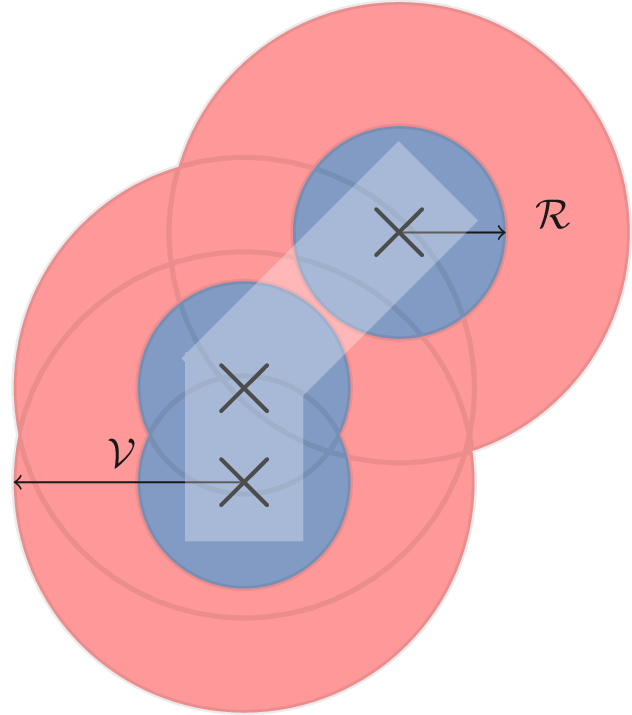}

\caption{3 agents moving in the space of $\mathcal X = \mathbb{R}^2$ with standard Euclidean distance. The bottom two agents potentially have an interdependent reward since they are within distance $\mathcal R$ of one another. Furthermore, every agent is within distance $\mathcal V$ of another agent so all agents can communicate with each other. Notably, the top and bottom agents may communicate even though they are not within distance $\mathcal V$ of each other.}
\label{imamdp_visual}
\end{figure}

\subsection{Objectives}
In this setting, we formally would like to answer the following questions that correspond to our objectives defined in \cref{contributions}:

\textit{1) Are there decentralized policies in this setting that perform well theoretically?}

Formally, we attempt to find group decentralized policies $\pi$ that maximize $V^\pi(s) = \mathbb{E}_{\tau\sim \pi\lvert_{s}}\bigg[\sum_{t = 0}^\infty \gamma^t r(s(t),a(t))\bigg]$. 

Equivalently, if we denote $\pi^*$ to be the optimal policy for all possible policies in the MDP and abbreviate $V^* = V^{\pi^*}$, we attempt to find group decentralized policies $\pi$ such that $\lvert V^*(s) - V^\pi(s)\rvert$ is small. This will establish concrete group decentralized policy solutions for Locally Interdependent Multi-Agent MDP's.

\textit{2) How significantly does decentralized partial observability impact the best possible performance?}

Formally, we ask whether $\lvert V^* (s) - \max_{\pi \in \Pi _{\mathcal V}} V^\pi (s)\rvert$ is small. This will answer whether the group decentralized policy class is a viable alternative to the entire policy class for optimization purposes.

In this paper, we will answer both of these questions simultaneously by providing three group decentralized policies with tight bounds.

\subsection{Properties}
Recalling $\mathcal V > \mathcal R$, a consequence of the reward and visibility structure is a buffer of time that agents can coordinate before the dependence begins. A constant that will occur frequently in this paper is $c = \lfloor \frac{\mathcal V-\mathcal R}{2}\rfloor$ and it represents the number of time steps that an agent from a visibility group cannot interact with agents from other visibility groups. This is more rigorously discussed in \cref{dtl}, the Dependence Time Lemma.

Further, the positive gap between the visibility and dependence constant can be used to create an equivalent representation on the reward function. Firstly, for any $g \subset \mathcal N$ define $r_g(s_g, a_g) = \sum_{j, k\in g} \overline r_{j,k}(s_j,a_j, s_k,a_k)$. 

Then since $\overline r_{j,k}(s_j, a_j, s_k, a_k) = 0$ when $d(s_j, s_k)\geq \mathcal V > \mathcal R$, we can say $r(s,a) = \sum_{z \in Z(s)} r_z(s_z, a_z)$. This equivalent decomposition for the reward function will also be used regularly.

\subsection{Applications}
In terms of modeling tasks in the environment, we may consider specific instances of the Locally Interdependent Multi-Agent MDP. We describe examples of three specific tasks below.

\subsubsection{Cooperative Navigation}

In cooperative navigation, agents need to navigate toward a location without colliding with other agents, where two agents collide if they enter into a physical ``width'' of each other \cite{lowe2017multi}.  Cooperative navigation problems are important in areas that have vehicles moving in a physical environment such as in autonomous driving, UAV's, robot navigation etc.

To model this, we can set $\mathcal R$ to be the maximum width of an agent and assign a large penalty when other agents enter their width radius, or something more intricate depending on the application. For each agent $k$, we may provide individual incentives (such as for navigating to a location) by properly designing $\overline{r}_k(s_k, a_k)$. Thus agents are incentivized to accumulate individual rewards $\overline{r}_k(s_k, a_k)$ (for navigating to a location) while avoiding penalties for colliding other agents. Each agent may also have differing dynamics $P_k$ depending on the physical environment. 

For a toy example of cooperative navigation using Locally Interdependent Multi-Agent MDP's see \cref{bullseye_many}.

\subsubsection{Obstacle Avoidance}
In the setting established in the previous section, we may also incorporate many types of dynamic obstacles by treating the obstacles as agents with trivial action spaces. 
For example, an obstacle $b\in \mathcal N$ would have a trivial action space $\mathcal A_b = \{X\}$ and dynamics $s'_b \sim P_b(\cdot \lvert s_b, X)$ that model the movement of the obstacle. We may then penalize other agents that are within $\mathcal R$ of agent $b$. For an example of a static obstacle, see \cref{highway}.

In the above setting, the other agents do not know the position of obstacle $b$ apriori, so the agents must adapt to it dynamically. This is in contrast to integrating an obstacle into the reward function by penalizing agents for approaching a point $x^*\in \mathcal X$. In this case, all agents would know the position of this obstacle at all times.

\subsubsection{Formation Control}
In some applications, we would like agents to accomplish a task such as cooperative navigation or obstacle avoidance but prefer a specific formation, such as for drag reduction and surveillance in UAV's \cite{wang2007cooperative,dong2014time}.

By providing a more intricate structure on $r$, we may specify a preferred relative position to each agent. Suppose for example that relative to agent $1$ at position $x_1\in \mathcal X$, we would prefer if agent $2$ and $3$ are at positions $x_2^*, x_3^* \in \mathcal X$ respectively where $d(x_1, x_2^*) \leq \mathcal R$ and $d(x_1, x_3^*) \leq \mathcal R$. For example, $x_2^*$, $x_3^*$ may form a specific angle or maintain a specific distance with $x_1$. Then for all possible internal states for the three agents, we may take the associated states $s_1, s_2^*, s_3^*$, and set $\overline r (s_1, a_1, s_2^*, a_2)$ and $\overline r (s_1, a_1, s_3^*, a_3)$ to be a positive quantity for any $a_1, a_2, a_3$ to incentivize these positions. In the case $\mathcal X$ is continuous (such as $\mathbb R^2$), these can be small positive reward regions that lie in a radius $\mathcal R$ of agent $1$. This type of reward structure in the Locally Interdependent Multi-Agent MDP can model complex formations that prefer specific agent positions relative to each other.

For an implementation of a toy example see \cref{lane_merging}.

\section{Main Results}
\label{sec:main_results}
Here we settle the case of whether the group decentralized policy class can perform well in the broadly applicable class of Locally Interdependent Multi-Agent MDP's. We answer our objectives by first providing three closed-form stationary group decentralized policies with associated upper bound guarantees. Then, we provide a lower bound guarantee that matches each upper bound up to constant factors. Lastly, as a corollary, we find that the performance of group decentralized policies improves exponentially with respect to the visibility.

All theorems are proved in \cref{proofs}.

\subsection{Three Upper Bound Constructions}
\subsubsection{Amalgam Policy}
With our motivation to find group decentralized policies that perform well theoretically for general Locally Interdependent Multi-Agent MDP's, we introduce our first construction, the Amalgam Policy $\lambda$. It is defined by taking the optimal policy for each visibility group. It is named as such because it is an amalgamation of local joint optimal policies. Formally $\lambda(s) = (\pi_z^*(s_z): \forall z \in Z(s))$ where $\pi_z^*$ is the optimal policy for $\mathcal M_z$ the Locally Interdependent Multi-Agent MDP defined for a subset of agents $z$. We show that it satisfies the guarantee:
\begin{theorem}
\label{amal_bound}
$\lvert V^*(s) - V^{\lambda} (s) \rvert \leq \frac{2}{(1 - \gamma)^2}\gamma^{c + 1}  \tilde{r}$.
\end{theorem}

Recalling the definition $c = \lfloor \frac{\mathcal V - \mathcal R}{2}\rfloor$, this shows the quality of the Amalgam Policy improves exponentially with the increase of the visibility. Storing this policy will in general require less space because it only requires storing $\pi^*_z(s_z)$ for states $s_z$ with agents $z$ in the same communication group. Computing this policy can be fully scalable to compute when group sizes do not exceed some constant $L$. A more in-depth discussion is provided in \cref{extensions}.
In the long horizon, this policy performs well for many of our examples described in \cref{simulations}.

In addition to the Amalgam Policy, we now introduce two more policies that approach the problem in unique ways and have some differing properties.

\subsubsection{Cutoff Policy}
\label{main_cutoff}

We now introduce the group decentralized Cutoff Policy $\chi$ which has properties that take advantage of the setting to improve scalability. We introduce a notion of the Cutoff Multi-Agent MDP which modifies the communication structure for Locally Interdependent Multi-Agent MDP's so agents are not allowed to re-enter visibility groups (for more details see \cref{cutoff}). The Cutoff Policy is the optimal policy in the Cutoff Multi-Agent MDP converted for use in the Locally Interdependent Multi-Agent MDP setting. Formally, we define $\chi(s) = (\pi_z^\mathcal C((s_z, \{z\})): \forall z \in Z(s)) $.

We prove the bound:
\begin{theorem}
\label{cutoff_bound}
$\lvert V^*(s) - V^{\chi} (s) \rvert \leq \frac{2 - \gamma}{(1 - \gamma)^2}\gamma^{c + 1}  \tilde{r}$.
\end{theorem}

Again, this is a group decentralized policy so it benefits from the reduced space required to store the policy. However, in contrast to the Amalgam Policy, the computational scalability described in \cref{decentralized_scalable} is achieved with the Cutoff Policy. It benefits from the partial observability because of the inherent properties of the Cutoff Multi-Agent MDP. Namely, the exact Bellman Equations in the Cutoff Multi-Agent MDP only require the computation of values for states $s_z$ where all agents in $z$ are in the same communication group (See \cref{cutoff}). This can significantly improve the tractability, as the computational scalability now reflects the size of the stored policy.  

The policy can perform better than the Amalgam Policy in the long horizon for examples with positive interdependent rewards (see \cref{aisle_walk}).

\subsubsection{First Step Finite Horizon Optimal Policy}
The final policy is called the First Step Finite Horizon Optimal Policy $\phi$. As the name suggests, the policy will take the first step in the discounted finite horizon optimal policy for the Locally Interdependent Multi-Agent MDP.  Formally, if $\{ \pi^{*, 0},  \pi^{*, 1}, ...,  \pi^{*, c}\}$ are the set of discounted finite horizon optimal policies with horizon $c$, the policy is defined as $\phi(s) = \pi^{*, 0}(s)$. This is stationary and shown to be a group decentralized policy by \cref{finite_group_decentralized}. For this policy, we prove:
\begin{theorem}
\label{finite_bound}
$\lvert V^*(s) - V^{\phi} (s) \rvert \leq \frac{2}{1 - \gamma} \gamma^{c + 1} \tilde{r}$.
\end{theorem}

We also show that the First Step Finite Horizon Optimal Policy can equivalently be computed on the Cutoff Multi-Agent MDP with \cref{cutoff_finite_proof}. Therefore, we obtain the same reduced computational and storage savings. However, since this policy only considers a limited horizon, the long horizon behavior can be inferior to the other two.

\subsubsection{Summary of the Three Policies} These three group decentralized policies take different intuitive approaches to perform well. Roughly, the Amalgam, Cutoff, and First Step Finite Horizon Optimal Policies convert the optimal policy $\pi^*$, optimal cutoff policy $\pi^{\mathcal C}$, and the optimal discounted finite horizon policy $\{ \pi^{*, 0},  \pi^{*, 1}, ...,  \pi^{*, c}\}$, respectively, into valid stationary group decentralized policies.

All of these policies are more tractable to store and compute than the optimal policy except for the Amalgam Policy which can be modified using our discussions in \cref{extensions} to improve computational scalability.

Although they all satisfy a similar theoretical guarantee, differences will show in varying long horizon behaviors that are explored in \cref{simulations}. 

\subsection{Lower Bound}
Next, we present a construction for a lower bound on the best possible group decentralized policy. 
\begin{theorem}
\label{lower_bound}
There exists a Locally Interdependent Multi-Agent MDP $\mathcal M(\ell)$ such that for every visibility $\mathcal V = 2\ell + 1$, $\ell \in \{0, 1, ...\}$, there exists $s,a$ such that $\lvert V^* (s) - \max_{\pi \in \Pi_ {\mathcal V}} V^\pi (s)\rvert \geq \frac{1}{2} \frac{\gamma^{c + 2}}{1 - \gamma} \tilde r$.
\end{theorem}
The above lower bound matches our three prior results up to constant factors and proves they nearly have the best performance guarantee we can establish for general Locally Interdependent Multi-Agent MDP's.
\subsection{Performance of Group Decentralized Policies}
To conclude, notice all three upper bounds are also upper bounds for $\lvert V^* (s) - \max_{\pi \in \Pi _{\mathcal V}} V^\pi (s)\rvert $. Together with our discussion above, we have proven a fundamental property of group decentralized policies in Locally Interdependent Multi-Agent MDP's. Namely,

\begin{corollary}
\label{thm:bigtheorem}
$\lvert V^* (s) - \max_{\pi \in \Pi _{\mathcal V}} V^\pi (s)\rvert  = O(\gamma^{\mathcal V})$.
\end{corollary}

We know this is the best asymptotic guarantee we may establish for general Locally Interdependent Multi-Agent MDP's because of our lower bound. This shows that the class of group decentralized policies can indeed produce viable solutions in Locally Interdependent Multi-Agent MDP's. 

\section{Proof Summary}\label{sec:proof_summary}

As mentioned previously, we would like to upper bound the quantity $\lvert V^*(s) -  V^{\pi}(s)\rvert$ for the three group decentralized policies in place of $\pi$. All three will use similar proof techniques that are outlined here.

The proofs for each policy in \cref{proofs} will be divided into 3 steps. In the first step, we discuss the consequences of the Dependence Time Lemma. In the second step, we will use ideas from the first step to demonstrate the performance of three naive versions of our closed-form group decentralized policies. Lastly, we will use the Telescoping Lemma to convert the guarantees from these naive policy candidates into guarantees for our closed-form group decentralized policies.

\subsection{Step 1: Dependence Time Lemma}
In this step of the proofs, we will introduce various relevant lemmas derived from the Dependence Time Lemma to be used in later steps. We introduce the Dependence Time Lemma here.

The lemma describes a fundamental property of visibility groups in Locally Interdependent Multi-Agent MDP's. Since $\mathcal V > \mathcal R$ and agents move at most a distance of $1$, a buffer of time $c = \lfloor \frac{\mathcal V - \mathcal R}{2}\rfloor$ is created during which an agent cannot interact with another agent from outside its visibility group. That is, the interdependent rewards $\overline r(s_j(t), a_j(t), s_k(t), a_k(t)) = 0$ for agents $j,k$ in different initial visibility groups for up to $c$ time steps. Formally, this is stated in the following lemma.
\begin{lemma}
\label{dtl}
(Dependence Time Lemma) For any realizable trajectory $(s(t), a(t))$ and time step $T$, we have  $$ r(s(T + \delta), a(T + \delta)) = \sum_{z \in Z(s(T))} r_z(s_z(T + \delta),a_z(T + \delta))$$ for $\delta \in \{0, ..., c\}$.
\end{lemma}
\begin{proof}
Consider two agents $j, k \in \mathcal N$ such that $d(s_j(T), s_k(T)) > \mathcal V$.
Firstly, using the definition of the distance metric and the probability transition function of the Locally Interdependent Multi-Agent MDP, we may conclude $d(s_j(T), s_j(T + \delta)) \leq \sum_{i = 0}^{\delta - 1} d(s_j(T + i), s_j(T + i + 1)) \leq \delta$. Similarly for agent $k$.
\\\\
Further by applying the reverse triangle inequality and substituting what we found, we have

\begin{align*}
    d(s_j(T &+ \delta), s_k(T + \delta)) \\
    &\geq d(s_j(T), s_k(T)) - d(s_j(T + \delta), s_j(T)) \\
    &\hspace{15ex}- d(s_k(T + \delta), s_k(T)) \\
    &> \mathcal V - 2 \delta \\
    &\geq \mathcal V - 2c \\
    &\geq \mathcal V - 2\bigg(\frac{\mathcal V-\mathcal R}{2}\bigg) \\
    &= \mathcal R.
\end{align*}
\\
Now that we established $d(s_j(T + \delta), s_k(T +\delta)) > \mathcal R$, by the definition of the reward function for the Locally Interdependent Multi-Agent MDP, $\overline r_{j, k} (s_j(T + \delta), a_j(T + \delta), s_k(T + \delta), a_k(T + \delta)) = 0$ for agents $j, k$ that have  $d(s_j(T), s_k(T)) > \mathcal V$. This implies $r(s(T + \delta), a(T + \delta)) = \sum_{z \in Z(s(T))} r_z(s_z(T + \delta),a_z(T + \delta))$.
\\
\end{proof}

\subsection{Step 2: Performance of Naive Policies}
With some consequences of the Dependence Time Lemma in hand, this step of the proof will look to three naive policy candidates that satisfy a guarantee similar to the one we are after. They are naive because none of them are valid group decentralized policies. We will attempt to convert these into valid group decentralized policies afterward.

The proofs for all three following lemmas are provided in \cref{proofs}.

The first policy corresponds to a naive version of the Amalgam Policy. For any partition $G$ on $\mathcal N$, let $\pi^*_G(s) = (\pi_g^*(s_g): \forall g \in G)$. Then the policy is defined as $\pi^*_{Z(s^*)}(s)$ for some fixed $s^*$. Notice that this is not a group decentralized policy and is not allowed in our setting because the communication groups are fixed. Using lemmas derived in step 1, we may derive the following lemma:
\begin{lemma}
\label{amal_naive}
$\lvert V^*(s) - V^{\pi^*_{Z(s)}} (s)\rvert \leq \frac{2}{(1 - \gamma)}\gamma^{c+1}  \tilde{r}$.
\end{lemma}

The second policy candidate is the optimal cutoff policy $\pi^{\mathcal C}$. This policy viewed from the perspective of the Locally Interdependent Multi-Agent MDP is non-Markovian since the policy uses information about the cutoff partition. However, the following result still holds:
\begin{lemma}
\label{cutoff_naive}
$\lvert V^*(s) - V^{\pi^{\mathcal C}} ((s, Z(s)))\rvert \leq \frac{1}{(1 - \gamma)}\gamma^{c+1} \tilde{r}$.
\end{lemma}

The last policy candidate is the discounted finite horizon optimal policy with a horizon of $c$. It will turn out that a consequence of the Dependence Time Lemma is that $\pi^{*,0}$ in $\{ \pi^{*, 0},  \pi^{*, 1}, ... ,  \pi^{*, c}\}$ is a group decentralized policy by \cref{finite_group_decentralized}. However, all together, the discounted finite horizon optimal policy is not valid because the policy is non-stationary. We show that the following relationship is satisfied:
\begin{lemma}
\label{finite_naive}
$\lvert V^*(s) - V_0^{*} (s)\rvert \leq \frac{1}{(1 - \gamma)}\gamma^{c+1} \tilde{r}$.
\end{lemma}
where $V_0^*$ is the optimal discounted finite horizon values up to horizon $c$.

\subsection{Step 3: Telescoping Lemma}
Lastly, we will use the Telescoping Lemma introduced here to help bridge the gap between our policy candidates and their corresponding valid group decentralized policies. Intuitively, it says that if we can express the expected discounted rewards within two stopping times as the expected difference between two discounted value functions, then we may express the whole value function as an expectation of a telescoping sum of these discounted value functions.

Below is the formal condition we would like to be held for the lemma:

\begin{condition}
\label{telescope_condition}
For a MDP, a given policy $\pi$, and a sequence of stopping times $0 = T_0 < T_1 < ...$ , we have a family of  value functions 
 parameterized by the state $\{V^{s}\}$, such that the following inequality is satisfied:
\\
$\medmath{\mathbb{E}_{\tau \sim \pi\lvert_{s}} \bigg[ \sum_{t = T_i}^{T_{i + 1} - 1} \gamma^t r(s(t), a(t))\bigg] \geq}$\\$\medmath{
  \mathbb{E}_{\tau \sim \pi\lvert_{s}} \bigg[ \gamma^{T_i}V^{s(T_i)}(s(T_i)) - \gamma^{T_{i + 1}} V^{s(T_i)}(s(T_{i + 1}))\bigg]}$.
\end{condition}
It turns out that this condition is satisfied with equality for the Amalgam policy using a change in the visibility groups as the stopping time and setting $V^{s(T_i)}(s) = V^{\pi^*_{Z(s(T_i))}}(s)$. The Cutoff Policy satisfies a very similar condition with equality by setting the stopping time to when any agent enters another agents visibility and  $V^{s(T_i)}(s) = V^{\pi^{\mathcal C}}((s, Z(s)))$. The First Step Finite Horizon Optimal policy will use every step as the stopping time and $V^{s(T_i)}(s) = V^{*}_0(s)$. The condition does not hold with equality in this case.

Below is the formal statement for the Telescoping Lemma. Note that the inequality holds with equality when the condition \cref{telescope_condition} is held with equality.

\begin{lemma}
\label{telescope}
(Telescoping Lemma) For any MDP, \\
suppose $0 = T_0 < T_1 < ...$ are stopping times such that \cref{telescope_condition} is satisfied. Then we have $V^\pi (s) \geq \mathbb{E}_{\tau \sim \pi\lvert_{s}}\bigg[V^s(s) - \sum_{i = 1}^\infty \gamma^{T_i} \Delta_{i}^\tau\bigg]$ where $\Delta_{i}^\tau = V^{s(T_{i - 1})}(s(T_i)) - V^{s(T_i)}(s(T_i))$.
\end{lemma}
\begin{proof}
By definition and assumption, we may show
{\small
\begin{align*}
&V^{\pi}(s) = \mathbb{E}_{\tau \sim \pi\lvert_{s}}\bigg[\sum_{t = 0}^\infty \gamma^t r(s(t), a(t))\bigg]\\
&= \sum_{i = 0}^\infty \mathbb{E}_{\tau \sim \pi\lvert_{s}}\bigg[\sum_{t = T_i}^{T_{i + 1} - 1} \gamma^t r(s(t), a(t))\bigg]\\
&\geq  \mathbb{E}_{\tau \sim \pi\lvert_{s}}\bigg[\sum_{i = 0}^\infty\bigg(\gamma^{T_i}V^{s(T_{i})}(s(T_i))\\
&\hspace{15 ex}- \gamma^{T_{i + 1}}V^{s(T_{i})}(s(T_{i+1}))\bigg)\bigg]\\
& = \mathbb{E}_{\tau \sim \pi\lvert_{s}}\bigg[V^{s}(s) - \sum_{i = 1}^\infty \gamma^{T_i} \Delta_{i}^\tau\bigg].
\end{align*}
}%
\end{proof}
Notice that in the case stopping times become infinite after a certain time, $\gamma^\infty = 0$, and the sum inside the expectation will be finite.

By bounding the quantity $\Delta_i^\tau$ for all three policies, we may use this lemma to effectively bound the difference between the value functions of each naive policy candidate and its corresponding valid group decentralized policies. This paired with our bound on the naive policy candidates will be sufficient for producing our desired upper bounds.

\section{Fully Scalable Extensions}
\label{extensions}

We discussed in \cref{decentralized_scalable} that for our setting, the restricted visibility can improve the computation and storage requirements of these group decentralized policies. Here, we present modifications of our policies that fully mitigate the effect of the curse of dimensionality and are completely scalable. 

As is currently presented, when the visibility groups become too large, the state and action space grow exponentially and the group decentralized policies can still be intractable to compute and store. The goal of the following extensions is to effectively deal with the situation where group sizes become too large. 

\subsection{Eliminating Large Groups}
\label{eliminate_large}
Depending on the specific Locally Interdependent Multi-Agent MDP instance and the initial state we are interested in, agents may be sparse in $\mathcal X$ when running our group decentralized policies. This may happen when agents within $\mathcal R$ of each other are penalized and agents are incentivized to pursue spread-out independent rewards. 
With this sparsity, visibility group sizes may naturally not exceed some small constant $L$ when running our group decentralized policies. Then, for our three policies, we would only need to compute them for $\mathcal N = L$ agents. This can result in extreme computational savings in practice. See \cref{bullseye_many} for a toy example of how we can tractably compute the exact Amalgam Policy for many agents with a fixed initial state.

Since we are running the exact policy, our bounds in \cref{sec:main_results} hold for this initial state while benefiting from the scalability.

\subsection{Reducing Visibility and Splitting Large Groups}
\label{cuts}
Suppose that a group decentralized policy starting at a particular initial state, the sizes of visibility groups do exceed $L$ but the dependence groups (similar to visibility groups but defined for radius $\mathcal R$) do not exceed size $L$. This may occur when $\mathcal V$ is much larger than $\mathcal R$. Then, we may reduce the complexity by artificially splitting the visibility groups into smaller groups that do not exceed size $L$. For the best possible split based on the visibility, we may reduce the visibility $\mathcal V$ 
 just until the point $\mathcal V_L$ (up to some precision) that visibility group sizes do not exceed $L$. Essentially, reducing the visibility makes the agents more ``sparse''. 

\subsection{Approximations of Large Groups}
Lastly, we may handle larger group sizes $> L$ by using empirical tools such as deep actor-critic methods \cite{lowe2017multi}. For the Amalgam Policy, we may use these methods as intended to find an approximation for the optimal policy for these large groups of agents. For the Cutoff and the First Step Finite Horizon Policies, empirical methods may need to be modified using the Bellman Optimality equations presented in \cref{cutoff}. 

This method is unique in that it combines theoretical and empirical methods. For the Amalgam Policy, in the case that the group sizes are small $\leq L$, the local optimal policies are taken. When group sizes become large and finding the exact optimal policy is difficult, we transition to our empirical methods. This method provides very scalable solutions for Locally Interdependent Multi-Agent MDP's.

\section{Conclusion and Future Works}
We have formulated a broadly applicable theoretical framework that models partially observable decentralized agents with dynamic local dependencies. We illustrated this using examples of applications such as cooperative navigation, obstacle avoidance, and formation control. We found three closed-form policy solutions (Amalgam Policy, Cutoff Policy, and First Step Finite Horizon Optimal Policy) which satisfied theoretical upper bounds that were optimal up to constant factors. We also discussed the improved scalability in storing and computing these policies. Then, we gave various extensions to further improve the scalability of these policies. Lastly, we provided simulations and investigated the long term behaviors of the policies.

Overall, we believe that the proof techniques utilized in this paper may be used for applications beyond the Multi-Agent MDP setting. For example, the Telescoping Lemma as stated applies to general MDP's and is not commonly seen or used. 

Empirically, the scalability, applicability, and long horizon behaviors of the Amalgam and Cutoff Policies seem promising (see \cref{simulations}) when considered with our extensions described in \cref{extensions}. Some work however is required to overcome limitations described in \cref{penalty_jittering}.

An interesting follow-up work would be methods for performing reinforcement learning in this context. We are also interested to see whether we can extend this work to the setting where agents only act based on agents within view rather than forming communication groups.

% Acknowledgements should only appear in the accepted version.
\section*{Acknowledgements}
This research was supported by NSF Grants 2154171, 2339112, CMU CyLab Seed Funding, C3 AI Institute. In addition, Alex Deweese is supported by Leo Finzi Memorial Fellowship in Electrical \& Computer Engineering.

\section*{Impact Statement}
This paper presents work whose goal is to advance the field of 
Machine Learning. There are many potential societal consequences 
of our work, none which we feel must be specifically highlighted here.

% In the unusual situation where you want a paper to appear in the
% references without citing it in the main text, use \nocite
%\nocite{langley00}

\bibliography{icml2024/locally_interdependent}

\bibliographystyle{icml2024}

%%%%%%%%%%%%%%%%%%%%%%%%%%%%%%%%%%%%%%%%%%%%%%%%%%%%%%%%%%%%%%%%%%%%%%%%%%%%%%%
%%%%%%%%%%%%%%%%%%%%%%%%%%%%%%%%%%%%%%%%%%%%%%%%%%%%%%%%%%%%%%%%%%%%%%%%%%%%%%%
% APPENDIX
%%%%%%%%%%%%%%%%%%%%%%%%%%%%%%%%%%%%%%%%%%%%%%%%%%%%%%%%%%%%%%%%%%%%%%%%%%%%%%%
%%%%%%%%%%%%%%%%%%%%%%%%%%%%%%%%%%%%%%%%%%%%%%%%%%%%%%%%%%%%%%%%%%%%%%%%%%%%%%%
\newpage
\appendix

\section{Long Horizon Simulations}
\label{simulations}
Our theory demonstrates that asymptotically, increasing the visibility of the Amalgam, Cutoff, and First Step Finite Horizon Optimal Policies will improve their quality exponentially. The First Step Finite Horizon Optimal Policy achieves this by considering the first $c$ iterations and ignoring rewards beyond distance $c$. However, in practice, we may be interested in the behavior of these policies far beyond $c$ iterations for fixed visibility $\mathcal V$ and initial state, such as when rewards are sparse and unavailable for many iterations past $c$ or when the visibility is small. Therefore, we would like to study the non-trivial long horizon behaviors of the Amalgam and Cutoff Policies.

For the above purpose, we will run toy grid-world simulations for various settings which will serve as a proof of concept for our policy constructions, demonstrate the applications of our Locally Interdependent Multi-Agent MDP settings, and uncover some peculiarities of our policies in the long horizon. We will then discuss various rollouts of the policies starting from a specific initial state in the various instances of the Locally Interdependent Multi-Agent MDP. 

In the figures, red will represent the optimal policy, blue will represent the Amalgam Policy, and green will represent the Cutoff Policy. The initial starting points are denoted by X's and self-loops labeled with numbers describe how many times agents stayed in that position. Spaces colored in green represent locations where agents may obtain independent rewards and red represents independent penalties.

All examples shown will have trivial internal states and deterministic transitions.

\subsection{Bullseye Problem}
\label{bullseye}
The purpose of the following example is to demonstrate our theory described in the paper, provide an example of cooperative navigation, and begin exploring the differences between the Amalgam and the Cutoff Policies. The outcome of the simulation is shown in \cref{bullseye_figure}.

For this problem, we have a central reward of $+100$ that agents obtain individually at the center of the bullseye shown in green. Agents that reach the bullseye will not interact with any other agents and will no longer obtain any other rewards (formally, they are sent down a long chain of states with no rewards). Furthermore, any agent that is within a radius $\mathcal R = 20$ of another agent is penalized with $-500$. Also, agents that move away from the bullseye are penalized with $-2$. The visibilities $\mathcal V$ are described beneath the figure and the discount factor is $\gamma = 0.9$. We consider 2 agents and put one of them 24 spaces on the left side of the bullseye and the other 25 spaces on the right side. The optimal policy is to wait for the closer agent to get close enough to the bullseye and trail behind. 

We see that the Amalgam Policy with $\mathcal V = 25$ takes a sub-optimal policy that backtracks as it notices the other agent approaching. As we increase the visibility, we find the number of backtracks decreases, and the quality of the value function improves very quickly.

We also see at the bottom, the Cutoff Policy does not perform well in this example. A common theme across these simulations will be that the Cutoff Policy performs poorly when the interdependent rewards are primarily penalties. Note that this does not conflict with our theoretical upper bounds and a more extensive discussion is given in \cref{penalty_jittering}.

\begin{figure}[h]

\includegraphics[width=0.5\textwidth]{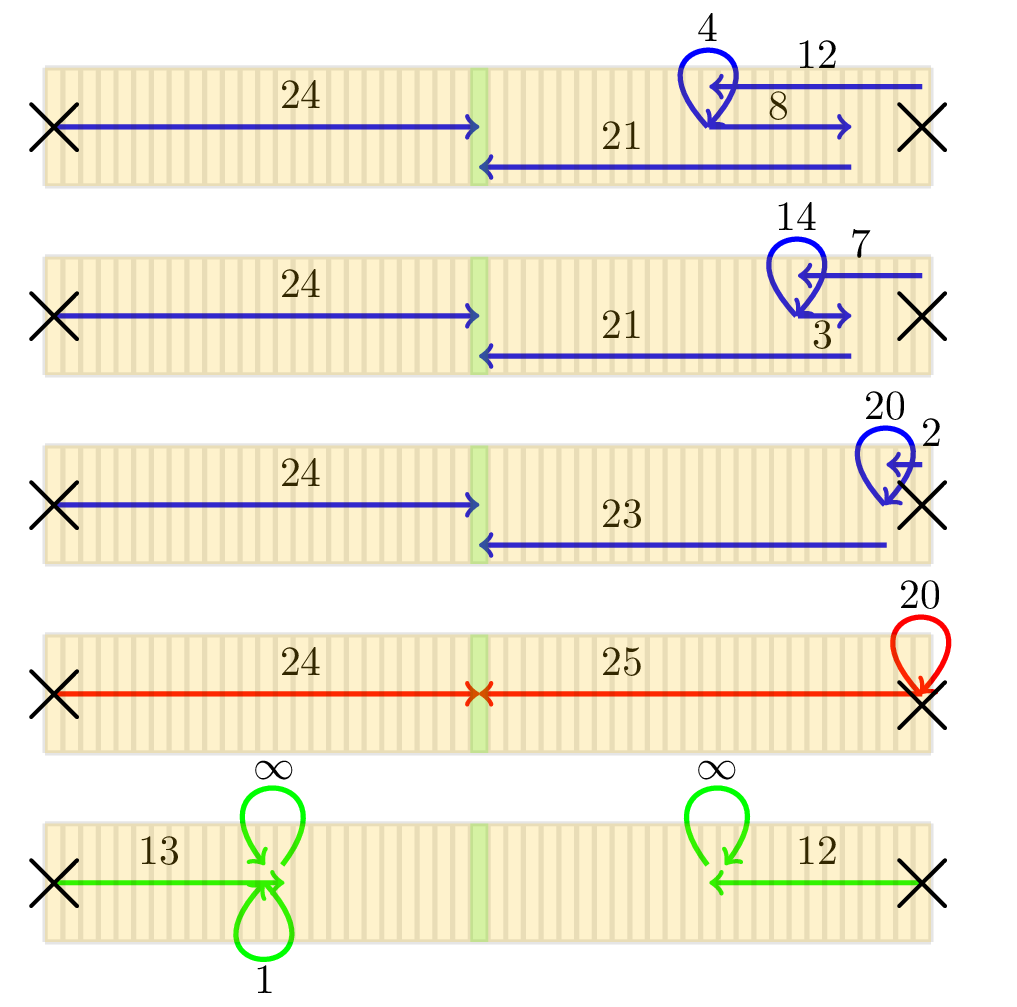}

\caption{\textit{Bullseye Problem:} In red is the optimal policy with a discounted sum of rewards of $8.85$. The top three in blue are Amalgam Policy rollouts with $\mathcal V=25$,$\mathcal \mathcal V=35$, $\mathcal V=45$ top to bottom. They have a total discounted reward of $6.74$, $8.26$, and $8.85$ respectively. Therefore, $ \lvert V^*(s) -  V^{\lambda}(s)\rvert$ is $2.11$, $0.59$, $0$ respectively.  In green is the Cutoff Policy with $\mathcal V = 25$.  It obtains a discounted reward of $-5.38$. All reported discounted sum of rewards are rounded to the second decimal place. }
\label{bullseye_figure}
\end{figure}

\subsection{Aisle Walk Problem}
\label{aisle_walk}
The following example demonstrates that the Cutoff Policy does indeed perform better than the Amalgam Policy for certain examples. This example will have positive interdependent rewards. The outcome of the simulation is shown in \cref{aisle_walk_figure}.

For this problem with 2 agents, with $\mathcal V = 2$, an interdependent reward of $20$ is given whenever the agents are within $\mathcal R = 1$ step of each other, so the agents are incentivized to stick together. However, there is a tempting reward of $+120$ on either side shown in green, that agents may split apart and obtain on the sides of the central aisle. Agents are required to move one step forward at each time step and can only change the column they are in to move out of the central aisle or return to the central aisle at specific points shown in the figure. At the top of the figure, the agents are stuck in those positions and left to interact interdependently for the remainder of the iterations. The discount factor is $\gamma = 0.9$. The optimal policy is then to split apart to obtain independent rewards and rejoin again at the end of the aisle. 

We see that, for this example, the Cutoff Policy performs better than the Amalgam Policy. The Amalgam Policy is tempted by the independent rewards and does not have the mechanism to rejoin again because the other agent is out of view. However, the Cutoff Policy takes this into account and decides that staying in the aisle is the best option. 

\begin{figure}[h]
\centering
\includegraphics[width=0.27\textwidth]{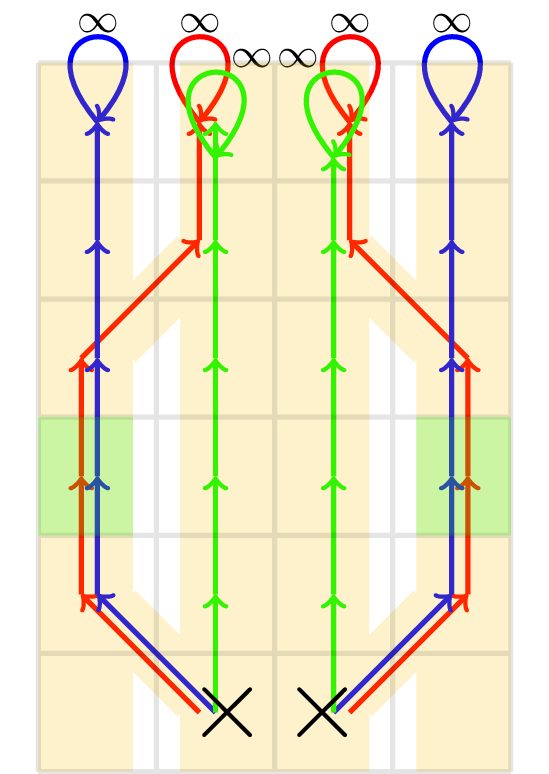}
\caption{\textit{Aisle Walk Problem:} In red is the optimal policy with a discounted reward of $496.84$, In blue is the Amalgam Policy with a discounted reward of $234.40$, and in green is the Cutoff Policy with a discounted reward of $400$. All reported discounted sum of rewards are rounded to the second decimal place. }
\label{aisle_walk_figure}
\end{figure}

\subsection{Highway Problem}
\label{highway}
This problem demonstrates obstacle avoidance in our setting and provides another example of the Cutoff Policy performing poorly when penalties are introduced in the interdependent rewards. The outcome of the simulation is shown in \cref{highway_figure_amalgam}.

Again with 2 agents, if any agent is within $\mathcal R = 3$ of another, the agent receives an interdependent penalty of $-500$. One agent is fixed near the center and acts as an obstacle. The visibility is $\mathcal V = 5$. 

There is a primary reward at the top left of $+100$ shown in green and agents that achieve the primary reward will no longer obtain any other rewards. The agent that is not fixed has the option of traveling around the obstacle agent to obtain the primary reward or to use the ``highway'' at the bottom left shown in red to transport there in 1 iteration but incur a penalty of $-25$. Using the highway will thus result in a less discounted primary reward. The discount factor is $\gamma = 0.98$. The optimal strategy is to pay the price and take the highway to the $+100$ reward because of the obstacle agent that is in the way. 

The Amalgam Policy has limited visibility and does not initially see the obstacle. Thus, it attempts to maximize its reward by avoiding the cost of the highway and taking the long way around. Once it notices the obstacle, it avoids it and eventually receives a greater discounted final reward that is suboptimal. 

The Cutoff Policy shown in \cref{highway_figure_cutoff} once again becomes confused once it notices the obstacle and performs poorly. This is again due to the penalties introduced in the interdependent rewards. A discussion of this phenomenon is provided in \cref{penalty_jittering}.

\begin{figure}[h]
\centering
\includegraphics[width=0.25\textwidth]{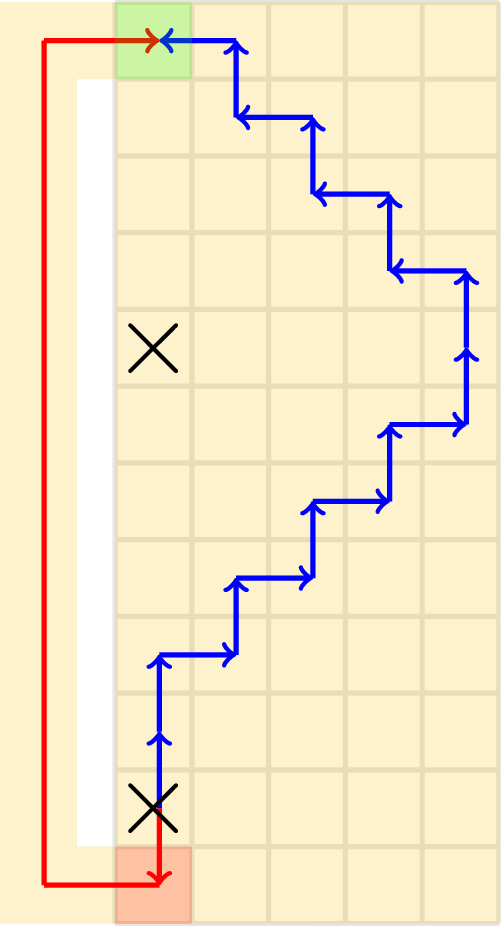}
    
\caption{\textit{Highway Problem with Amalgam and Optimal Policy:} In red is the optimal policy with a discounted reward of 73.5 and in blue is the  Amalgam Policy with 70.93 rounded to the second decimal place.}
\label{highway_figure_amalgam}
\end{figure}

\begin{figure}[h]
\centering
\includegraphics[width=0.28\textwidth]{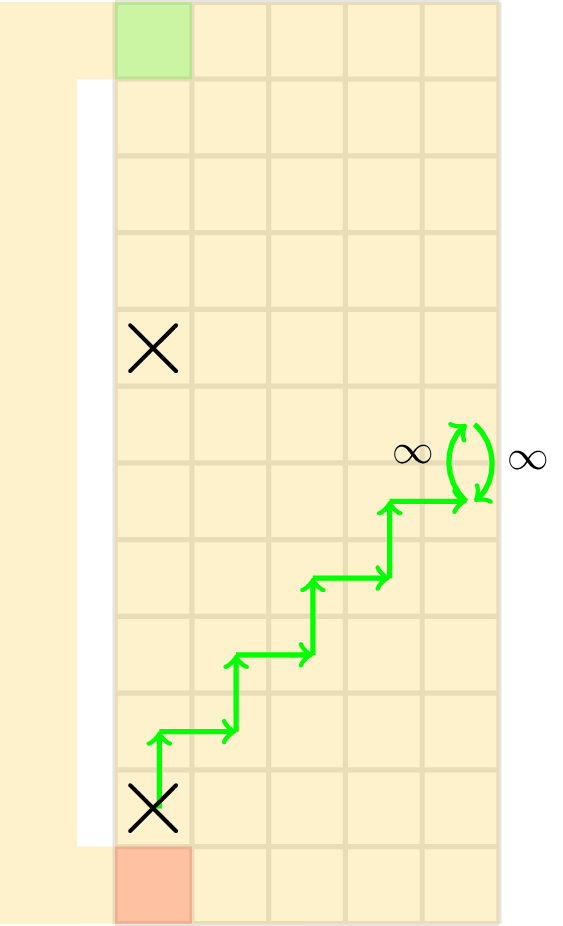}
    
\caption{\textit{Highway Problem with Cutoff Policy:} In green is the Cutoff Policy with an accumulated discounted reward of $0$.}
\label{highway_figure_cutoff}
\end{figure}

\subsection{Lane Merging Problem} 
\label{lane_merging}
The purpose of this example is to provide an example of formation control and to show a situation that both policies find the optimal policy. The outcome of the simulation is shown in \cref{lane_merging_figure}.

For this example, every agent can only move forward or stay in its position. If an agent is within 1 distance of another agent, it will incur a penalty of $-500$. However if the agent are exactly $\mathcal R = 2$ distance from another agent, it will receive a reward of $+10$. Lastly, any agent in the rightmost 7 squares will receive a reward of $+100$ at every step. $\mathcal V = 4$ and the discount factor is $\gamma = 0.9$. There are two agents on each side of two lanes that merge in the center, one set closer to the junction than the other. The optimal strategy is then to allow the agents closer to the junction to pass and then to follow in formation.

All policies for this example find the exact optimal policy. Once the agents take a step forward, the agents all come within view of each other and they can coordinate effectively.

\begin{figure}[h]
\centering
\includegraphics[width=0.52\textwidth]{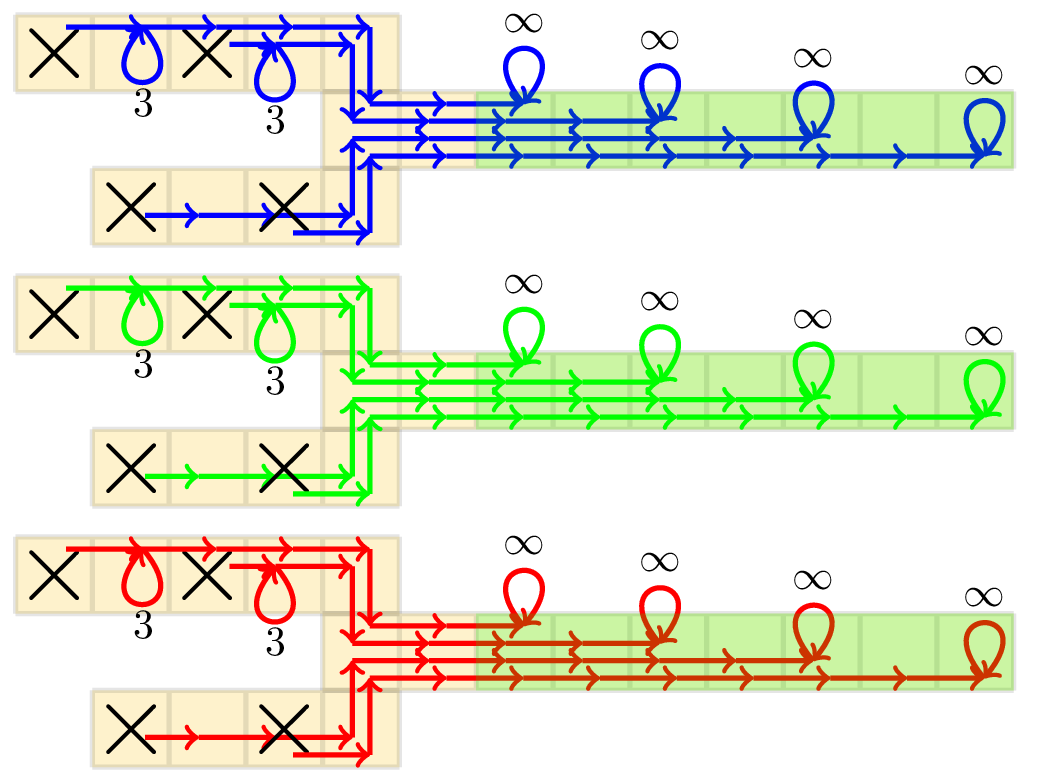}
\caption{\textit{Lane Merging Problem:} The Amalgam, optimal, and Cutoff Policies have the same trajectory with accumulated discounted reward of $2514.11$.}
\label{lane_merging_figure}
\end{figure}

\subsection{Bullseye Problem with Many Agents}
\label{bullseye_many}
This example demonstrates scalable cooperative navigation with many agents. The outcome of the simulation is shown in \cref{bullseye_many_figure}.

For this example, we have two ``bullseyes'' similar to \cref{bullseye}. Agents that are at the center of a bullseye will receive $+100$ reward and receive a penalty of $-10$ for moving away from either bullseye. Agents will also receive a penalty of $-500$ if they are within $\mathcal R = 1$ space of another agent. $\mathcal V = 3$ and the discount factor is $\gamma = 0.9$.  Similar to the Bullseye Problem, once an agent reaches the bullseye, it will not incur any more rewards or penalties. Here we consider an initial state with 8 agents. The optimal strategy is intractable to compute on the machine this was computed on.

For this particular initial state, group sizes of the Amalgam Policy do not exceed 3 so the exact actions were found by simply computing the Amalgam Policy for 3 agents. This is an example of how our discussion in \cref{eliminate_large} can be used in practice.

\begin{figure}[h]
\centering
\includegraphics[width=0.5\textwidth]{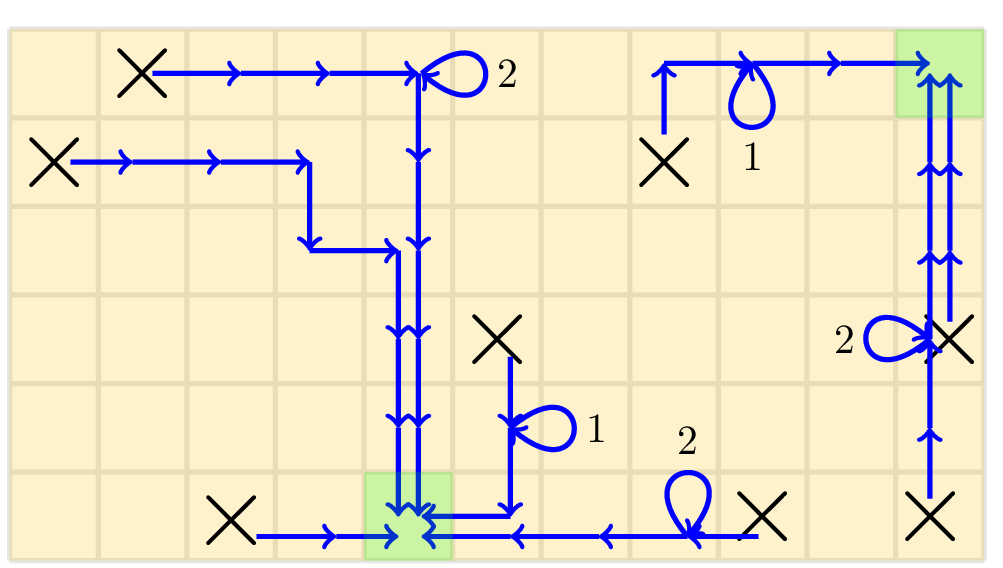}

\caption{\textit{Bullseye Problem with Many Agents:} The paths taken by the 8 agents with the Amalgam Policy. Only computation for 3 agents is required. }
 
\label{bullseye_many_figure}
\end{figure}

\subsection{Discussion: Penalty Jittering}
\label{penalty_jittering}
Intuitively the Cutoff Policy seems like it should perform better than the Amalgam policy for many cases as in the Aisle Walk Problem because the solution to the Cutoff Multi-Agent MDP incorporates information about the visibility dynamics. This is in comparison to the Amalgam Policy which simply consists of local optimal policies patched together.

But, as seen in the previous examples, if interdependent rewards include heavy penalties, the Cutoff Policy may not perform well. We call this phenomenon ``penalty jittering'' as the interdependent penalties cause agents taking the policy to get stuck moving back and forth. 
This is a general phenomenon that may occur for the Amalgam Policy as well but is particularly an issue for the Cutoff Policy. 

To illustrate the issue, consider a grid world in \cref{penalty_jit_fig}. Agents incur a large penalty (e.g. $-500$) when overlapping with each other ($\mathcal R = 0$). There is a large reward (e.g. $+100$) for remaining in the leftmost state and a small reward (e.g. $+10$) for remaining in the rightmost state. The optimal strategy is for one agent to stay in the leftmost state and the other agent to stay in the rightmost state. The visibility is $\mathcal V = 1$.

For this example, both the Amalgam and the Cutoff Policies will fail to find the optimal policy. When the agent on the right is in the rightmost state, the agent does not see the agent on the left but is aware of the high reward leftmost state so it moves to the center.  When the agent on the right is in the center and in the visibility of the left agent, both policies will suggest the right agent to move to the smaller reward rightmost state because the agent on the left already occupies the leftmost state. The agent will move to the rightmost state but will forget the existence of the other agent and re-enter the center state for the same reason as before. Therefore, the agent on the right will continue to move back and forth.

In general, this is a larger issue for the Cutoff Policy, because the ``suggestion'' made by the policy when the agents are connected assumes that when agents are disconnected, they will never reconnect (see \cref{cutoff}). That is, the initial suggestion made by the policy for the agent to leave the visibility group assumes the agent will be able to return to the area and obtain a high independent reward without any consequences. Therefore, the Cutoff Policy will tend to prefer these types of situations.

For this example, when overlapping agents gain a positive reward, the penalty jittering phenomenon will not be observed because if the policy ``suggests'' for an agent to leave the visibility, returning to the group will only seem less attractive with the interdependent rewards removed.

We believe that this ``forgetfulness'' is a general issue for group decentralized policies and overcoming this will be an important part of using these policies in more complex systems. One potential way to overcome this is to incorporate memory into  group decentralized policies, which we leave as a future direction.

Finally, we note that the penalty jittering does not conflict with our theoretical upper bounds, as the bounds are asymptotic in terms of visibility, whereas the penalty jittering examples occur in this long horizon setting with a fixed limited visibility.

\begin{figure}[h]

\centering
\includegraphics[width=0.2\textwidth]{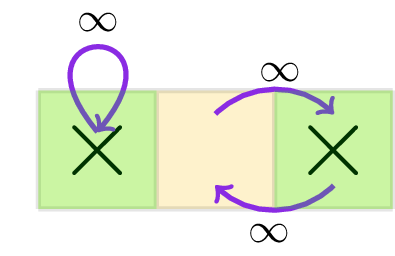}
\caption{Example of Penalty Jittering.}
\label{penalty_jit_fig}
\end{figure}

\section{The Cutoff Multi-Agent MDP setting}
\label{cutoff}
It will turn out that for any Locally Interdependent Multi-Agent MDP, there will be an instance of a related setting we call the Cutoff Multi-Agent MDP that has nice properties we describe below. It is used in the construction of one of our closed-form policies, the Cutoff Policy (See \cref{cutoff_bound}).

Intuitively, the Cutoff Multi-Agent MDP is a version of the Locally Interdependent Multi-Agent MDP with an embedded communication structure where agents do not interact with one another once they disconnect. More specifically, if two agents lie in different visibility partitions at any time, the agents will lie in different visibility partitions for all subsequent time steps (see \cref{cutoff_fig}). We refer to this new notion of a permanently disconnecting visibility partition as the cutoff partition.

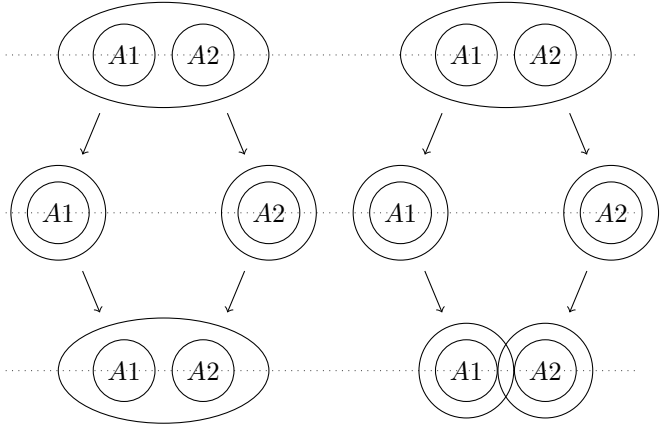
\begin{figure}[h]
\centering
\begin{tikzpicture}[scale =0.7]

\node[shape=circle,draw=black] (s1) at (-0.75,0) {$A1$};
\node[shape=circle,draw=black] (s2) at (-2,3){$A1$};
\node[shape=circle,draw=black] (s3) at (-0.75,6){$A1$};

\node[shape=circle,draw=black] (s4) at (0.75,0){$A2$};
\node[shape=circle,draw=black] (s5) at (2,3){$A2$};
\node[shape=circle,draw=black] (s6) at (0.75,6){$A2$};

\draw (0,0) ellipse (2 and 1);

\draw[](2,3)circle (0.9);
\draw[](-2,3)circle (0.9);

\draw (0,6) ellipse (2 and 1);
    
 \path [->, shorten >=12pt,shorten <=12pt] (s2) edge node[left] {} (s1);
 \path [->, shorten >=12pt,shorten <=12pt] (s3) edge node[left] {} (s2);
 
 \path [->, shorten >=12pt,shorten <=12pt] (s5) edge node[left] {} (s4);
 \path [->, shorten >=12pt,shorten <=12pt] (s6) edge node[left] {} (s5);

\node[shape=circle,draw=black] (s7) at (-0.75 + 6.5,0) {$A1$};
\node[shape=circle,draw=black] (s8) at (-2 + 6.5,3){$A1$};
\node[shape=circle,draw=black] (s9) at (-0.75 + 6.5,6){$A1$};

\node[shape=circle,draw=black] (s10) at (0.75 + 6.5,0){$A2$};
\node[shape=circle,draw=black] (s11) at (2 + 6.5,3){$A2$};
\node[shape=circle,draw=black] (s12) at (0.75 + 6.5,6){$A2$};

\draw[](2 + 6.5,3)circle (0.9);
\draw[](-2 + 6.5,3)circle (0.9);

\draw (0 + 6.5,6) ellipse (2 and 1);

\draw[](-0.75 + 6.5,0)circle (0.9);
\draw[](0.75 + 6.5,0) circle (0.9);
    
 \path [->, shorten >=12pt,shorten <=12pt] (s8) edge node[left] {} (s7);
 \path [->, shorten >=12pt,shorten <=12pt] (s9) edge node[left] {} (s8);
 
 \path [->, shorten >=12pt,shorten <=12pt] (s11) edge node[left] {} (s10);
 \path [->, shorten >=12pt,shorten <=12pt] (s12) edge node[left] {} (s11);

 \draw [draw=black, opacity=0.7, dotted] ( -3 , 6 ) -- ( 9 , 6 );
 \draw [draw=black, opacity=0.7, dotted] ( -3 , 3 ) -- ( 9 , 3 );
 \draw [draw=black, opacity=0.7, dotted] ( -3 , 0 ) -- ( 9 , 0 );

\end{tikzpicture}

\caption{The diagram on the left represents our communication structure in the Locally Interdependent Multi-Agent MDP. Agents that leave each other's visibility reconnect when re-entering each other's visibility. On the right is the Cutoff Multi-Agent MDP. Agents that leave each other's visibility are not able to reconnect even when the agents re-enter each other's visibility radii. }
\label{cutoff_fig}
\end{figure}

Formally for any realizable trajectory $\tau = (s(t), a(t))$ in a Locally Interdependent Multi-Agent MDP, we define the cutoff partition at time $T$ as $C^\tau(T) = \bigcap_{0 \leq t' \leq T} Z(s(t'))$. Here, the intersection denotes an intersection of partitions defined as $P_1 \cap P_2 = 
\{ p_1 \cap p_2 \lvert p_1 \in P_1, p_2 \in P_2\}\setminus \{\varnothing\}$ for partitions $P_1$, $P_2$.
Observe that this partition only gets finer in time similar to our intuition above.

Furthermore, in this setting, we will say that agents in different cutoff partitions do not incur interdependent rewards.

We can summarize these ideas with the complete definition of the Cutoff Multi-Agent MDP. For the following definition let $B(\mathcal N)$ be the set of all partitions on $\mathcal N$.

For any Locally Interdependent Multi-Agent MDP $\mathcal M = (\mathcal{S}, \mathcal A, P, r, \gamma)$, we can define the Cutoff Multi-Agent MDP as $\mathcal{C} = (\mathcal{S}^\mathcal C, \mathcal A^\mathcal C, P^\mathcal C, r^\mathcal C, \gamma)$.
Where:
\begin{itemize}
\item $\mathcal S^\mathcal C := \{(s, C) : s \in \mathcal S, C \in B(\mathcal N)\}$ 
\item $\mathcal A^\mathcal C := \mathcal A$
\item $ P^\mathcal C((s', Z(s') \cap C) \lvert (s, C), a) = P(s' \lvert s,a)$ and $0$ otherwise
\item $ r^\mathcal C((s, C), a) = \sum_{c \in C}r_c(s_c,a_c)$.
\end{itemize}

We will overload notation and say a policy $\pi$ in the Locally Interdependent Multi-Agent MDP may be viewed as a policy in the Cutoff Multi-Agent MDP by letting $\pi((s, C)) = \pi(s)$

Consider the relationship between the distribution of the trajectories with policy $\pi$ in the associated Locally Interdependent Multi-Agent MDP and in the Cutoff Multi-Agent MDP. We can compare realizable trajectories in the Locally Interdependent Multi-Agent MDP of the form $\tau = (s(t), a(t))$ with $s(0) = s$, $a(0) = a$ and realizable trajectories in the Cutoff Multi-Agent MDP of the form $\tau' = ((s(t), C'(t)), a(t))$ with $C'(0) = Z(s)$. Notice that by the definition of $P^\mathcal C$, we have for any time step $T$ and state action pair $s^*, a^*$,  

$\medmath{P\bigg((s(T), a(T)) = (s^*, a^*)\bigg) = }$
$\medmath{\hspace{10ex}
P^{\mathcal C}\bigg(((s(T), C'(T)), a(T)) = ((s^*, C^\tau(T)), a^*)\bigg)}$ 

where $C^\tau(T)$ is the Cutoff Partition for the Locally Interdependent Multi-Agent MDP trajectory. Thus, these trajectories are equivalently distributed.

For ease of notation going forward, anytime a state is a tuple consisting of a partition on the agents, it will refer to a state in the Cutoff Multi-Agent MDP. We will also use $\mathcal C _g$ to reference the related Cutoff Multi-Agent MDP for $\mathcal M_g$. The value function of $\mathcal C_g$ will be defined implicitly by passing in the state for a subset of agents $g\in \mathcal N$. For example, with a policy $\pi_g$, the value function at $(s_g, C_g)$ would be $V^{\pi_g}((s_g, C_g))$. We will also denote the optimal policy for $\mathcal C _g$ as $\pi^\mathcal C_g$.

The permanent disconnections in this setting lead to a key property for Cutoff Multi-Agent MDP's. Namely, the value function will decompose according to our agent partitions for the Cutoff Multi-Agent MDP.

We may equivalently write,

\begin{align*}
&V^{\pi_g}((s_g, C_g)) \\
 &= \mathbb{E}_{\tau\sim \pi_g\lvert (s_g, C_g)} \bigg[ \sum_{t = 0}^\infty \gamma^t \sum_{c \in C_g(t)} r_c(s_c(t), a_c(t))\bigg]\\
 &= \mathbb{E}_{\tau\sim \pi_g\lvert (s_g, C_g)} \bigg[ \sum_{t = 0}^\infty \gamma^t \sum_{c \in C_g} r_c(s_c(t), a_c(t))\bigg]\\
 &= \sum_{c \in C_g}\mathbb{E}_{\tau\sim \pi_g\lvert (s_g, C_g)} \bigg[ \sum_{t = 0}^\infty \gamma^t  r_c(s_c(t), a_c(t))\bigg]\\
 &= \sum_{c \in C_g} V^{\pi_c}((s_c, \{c\})).
\end{align*}

The second equality is permitted because the partitions only get finer with time. 

This decomposition tells us that the value function for states with the trivial partition (a partition with only a single group)
serve as the ``atoms" for the Cutoff Multi-Agent MDP. That is finding the value function for all states reduces to finding the value function for states with the trivial partition.

Furthermore, we can substitute this decomposition into the Bellman Consistency Equations to obtain: 
$V^{\pi_g}((s_g, \{g\})) = \mathbb{E} _{a_g\sim \pi_g(s_g)}\bigg[ Q^{\pi_g}((s_g, \{g\}), a_g)\bigg]$\\\\
where \\
$Q^{\pi_g}((s_g, \{g\}), a_g) = r((s_g, \{g\}), a_g) + $\\
$\hspace*{10ex}\gamma \mathbb{E} _{s'_g\sim P(\cdot \lvert s_g, a_g)}\bigg[ \sum_{z \in Z(s'_g)}V^{\pi_z}((s_z, \{z\}))\bigg]$.
\\\\
And in the case of the Bellman Optimality equations we have:
\\\\
$V^{*}((s_g, \{g\})) = \max_{a_g} Q^{*}((s_g, \{g\}), a_g)$
\\\\
where \\
$Q^{*}((s_g, \{g\}), a_g) = r((s_g, \{g\}), a_g) + $\\
$\hspace*{10ex}\gamma \mathbb{E} _{s'_g\sim P(\cdot \lvert s_g, a_g)}\bigg[ \sum_{z \in Z(s'_g)} V^{*}((s_z, \{z\}))\bigg]$.
\\\\
In other words, the process for deriving the values for trivial partition states only requires values of other trivial partition states. With this, we may safely ignore the value function for states with a non-trivial partition for the Cutoff Multi-Agent MDP.

In fact, for this paper, we will only be interested in values of the form $V((s, Z(s)))$. Consequentially, we only need to look at values of the form $V((s_g, \{g\}))$ for a subset of agents $g\in \mathcal N$ with $g \in Z(s)$ for some state $s$. This means our ``atoms'' are states $s_g$ that have every agent within distance $\mathcal V$ of some other agent, paired with the trivial partition. Thus, we have savings in this partially observable setting by only having to compute a smaller table of values corresponding to the states of agents within the same visibility group. 

By definition, these savings will show up exactly in the computation of the Cutoff Policy described in \cref{main_cutoff}.

\section{Proofs}
\label{proofs}
Recall from the proof summary in \cref{sec:proof_summary}, there will be three steps for proving each upper bound. 
\begin{itemize}
\item\textit{Step 1) Consequences of the Dependence Time Lemma}
\item\textit{Step 2) Bounding performance of naive policies}
\item\textit{Step 3) Use Telescoping Lemma to bound performance of final policy.}
\end{itemize}
We cover these 3 steps for each closed-form policy. Then, we conclude by presenting the construction and proof for the lower bound.
\\\\
For the remainder of this section, for a fixed partition $G$, define:\\
$V_{G}^\pi(s') = \mathbb{E}_{\tau \sim \pi\lvert_{s'}}\bigg[\sum_{t = 0}^\infty \sum_{g \in G}\gamma^t r_g(s_g(t), a_g(t))\bigg]$.  

We will also denote for any partition on $\mathcal N$, $E(P) = \{(j,k) \lvert j,k \in p, p \in P\}$ to be pairs of agents within the same partition. Similarly, $E^c(P) = \mathcal N \times \mathcal N \setminus E(P)$ are pairs of agents in different partitions.

\subsection{The Amalgam Policy}
Recall the definition of the Amalgam Policy $\lambda(s) = (\pi_z^*(s_z): \forall z \in Z(s))$.

\subsubsection{Step 1)}

Using the Dependence Time Lemma, we may derive the following lemma. Intuitively, it is the effect that the Dependence Time Lemma has on the Q-values for Locally Interdependent Multi-Agent MDP's.
\begin{lemma}
\label{dtl_q}
For any policy $\pi$, any realizable trajectory $(s(t), a(t))$ , and $\delta \in \{0, ..., c\}$, let $T \geq \delta$ be an arbitrary time step. We have $\lvert V^\pi(s(T)) - V_{Z(s(T- \delta))}^\pi(s(T)) \rvert \leq \frac{\gamma^{c + 1 - \delta}}{1 - \gamma}\tilde r.$\\
\end{lemma}
\begin{proof}
Recall, in our notation $\tau \sim \pi\lvert_{s,a}$ refers to the random trajectory $(s(t), a(t))$ generated by the MDP starting with $s(0) = s, a(0) = a$.

Notice that if we are given a new trajectory realization $(s'(t'), a'(t'))$ with $s'(0) = s(T)$, the concatenated trajectory defined as  $(s(t''), a(t''))$ for $t'' \in \{0,...,T - 1 \}$ and $(s'(t'' - T), a'(t'' - T))$ for $t'' \in \{T , T + 1, ...\}$ is also a realizable trajectory. Furthermore, we may apply Dependence Time Lemma on this trajectory at time step $T - \delta$ to prove that $r(s'(t'), a'(t')) = \sum_{z \in Z(s(T-\delta))}\overline r_z(s_z'(t'), a_z'(t'))$ for $t'\in\{0, ..., c - \delta\}$. Using this decomposition, we can obtain:

{\scriptsize
\begin{align*} 
 &V^{\pi} (s(T)) \\
 &= \mathbb{E}_{\tau \sim \pi\lvert_{s(T)}}\bigg[\sum_{t' = 0}^\infty \gamma^{t'} \sum_{j, k \in \mathcal N}\overline r_{j, k}(s_{j}'(t'), a_j'(t'), s_{k}'(t'), a_k'(t'))\bigg]\\
 &= \mathbb{E}_{\tau \sim \pi\lvert_{s(T)}}\bigg[\sum_{t' = 0}^{c - \delta} \gamma^{t'} \sum_{z \in Z(s(T - \delta))} r_z(s_z'(t'), a_z'(t')) \\
 & \hspace{15ex}+ \sum_{t' = c + 1 - \delta}^\infty \gamma^{t'}\sum_{j, k \in \mathcal N} \overline r_{j, k}(s_{i}'(t'), a_j'(t'), s_{k}'(t'), a_k'(t'))\bigg]\\
 &= \mathbb{E}_{\tau \sim \pi\lvert_{s(T)}}\bigg[\sum_{t' = 0}^\infty \gamma^{t'} \sum_{z \in Z(s(T - \delta))} r_z(s_z'(t'), a_z'(t')) \\
 & \hspace{0ex} + \sum_{t' = c + 1 - \delta}^\infty \gamma^{t'} \sum_{(j,k) \in E^c(Z(s(T - \delta)))}\overline r_{j,k}(s'_j(t'), a'_j(t'), s'_k(t'), a'_k(t'))\bigg]\\
 &:= V_{Z(s(T - \delta))}^\pi(s(T)) + \xi.
\end{align*}
}%
For the third equality, recall that by our notation, $E^c(Z(s(T - \delta)))$ denotes tuples of two agents $j,k$ that do not lie in the same partitions of $Z(s(T - \delta))$. In this step, we have completed the infinite sum in the first term by regrouping the $\overline r_{j,k}$ terms.

Which concludes the proof since $\xi \in [-\frac{\gamma^{c + 1 - \delta}}{1 - \gamma}\tilde r, \frac{\gamma^{c + 1 - \delta}}{1 - \gamma}\tilde r]$.
\\
\end{proof}

\subsubsection{Step 2)}
Recall the definition of the naive policy candidate 
 $\pi^*_{Z(s^*)}(s) = (\pi_g^*(s_g): \forall g \in Z(s^*))$ for some fixed $s^*$.

\textbf{Proof for \cref{amal_naive}}. With the lemma from step 1 in hand, we are ready to prove the bound on the naive policy candidate.
\begin{proof}
Notice that because $\pi^*_{Z(s)}$ takes the local optimal policies, the expected discounted sum of rewards accumulated by the agents in this group is optimal. Therefore, for any $z \in Z(s)$,
\\
 $\mathbb{E}_{\tau \sim \pi^*_{Z(s)}\lvert_{s}}\bigg[\sum_{t = 0}^\infty \gamma^{t} r_{z}(s_{z}(t), a_z(t))\bigg] \geq$\\
 $ \mathbb{E}_{\tau \sim \pi^*\lvert_{s}}\bigg[\sum_{t = 0}^\infty \gamma^{t} r_{z}(s_{z}(t), a_z(t))\bigg] $.
 \\
 This implies $V^{*}_{Z(s)}(s) - V^{\pi^*_{Z(s)}}_{Z(s)}(s) \leq 0$.
 \\\\
 Using this relationship between the value functions and applying \cref{dtl_q} twice with $\delta = 0$,
 
\begin{align*}
&V^{*}(s) - V^{\pi^*_{Z(s)}}(s) \\
&\leq V^{*}_{Z(s)}(s) - V^{\pi^*_{Z(s)}}_{Z(s)}(s) 
 + \frac{2}{(1 - \gamma)}\gamma^{c + 1}  \tilde{r}\\
&\leq \frac{2}{(1 - \gamma)}\gamma^{c + 1}  \tilde{r} .
\end{align*}
\end{proof}

\subsubsection{Step 3)}

We start with an auxiliary lemma before proving \cref{amal_bound}.
\begin{lemma}
\label{amalgam_condition}
For realizable trajectories in a Locally Interdependent Multi-Agent MDP, define a sequence of stopping times with $T_0 = 0$ and $T_i$ the time step $t$ for which $Z(s(t)) \neq Z(s(t - 1))$ for the $i$-th time. Then \cref{telescope_condition} holds for the Amalgam Policy $\lambda$ with $V^{\pi^*_{Z(s)}}$ as the family of value functions.
\end{lemma}
\begin{proof}
 For this proof, given a realizable trajectory $\tau = (s(t), a(t))$, we let $\tau_T$ be $(s(t), a(t))$ with $ t \in \{T, ... \infty\}$. $\tau^T$ will be $(s(t), a(t)) $ for $t \in \{0, ... T - 1\}$ appended with $s(T)$. Lastly, $\tau_{T_1}^{T_2}$ is $ (s(t), a(t)) $ for $t \in \{T_1, ..., T_2 -1\}$ appended with $s(T_2)$.
 
To begin the proof, we show the following equivalence.
{\small
\begin{align*} 
&\mathbb{E}_{\tau \sim \lambda\lvert_{s}} \bigg[ \sum_{t = T_i}^{T_{i + 1} - 1} \gamma^t r(s(t), a(t))\bigg] \\
& = \mathbb{E}_{\tau^{T_i} \sim \lambda\lvert_{s}}\bigg[\mathbb{E}_{\tau \sim \lambda\lvert_{s}} \bigg[ \sum_{t = T_i}^{T_{i + 1} - 1} \gamma^t r(s(t), a(t)) \bigg| \tau^{T_i}\bigg] \bigg]\\
& = \mathbb{E}_{\tau^{T_i} \sim \lambda\lvert_{s}}\bigg[\mathbb{E}_{\tau_{T_i} \sim \lambda\lvert_{s}} \bigg[ \sum_{t = T_i}^{T_{i + 1} - 1} \gamma^t r(s(t), a(t)) \bigg| \tau^{T_i}\bigg] \bigg]\\
& = \mathbb{E}_{s(T_i) \sim \lambda\lvert_{s}}\bigg[\mathbb{E}_{\tau_{T_i} \sim \lambda\lvert_{s}} \bigg[ \sum_{t = T_i}^{T_{i + 1} - 1} \gamma^t r(s(t), a(t)) \bigg| s(T_i)\bigg] \bigg].\\
\end{align*}
}%
The second equality eliminates unnecessary random variables that don't appear in the expectation and the third equality evokes the strong Markov property of Markov Chains. 

Now notice that by the definition of our stopping times, for all $t' \in \{T_i, ..., T_{i + 1} - 1\}$ we have $Z(s(T_i)) = Z(s(t'))$ and therefore $\pi^*_{Z(s(T_i))} = \pi^*_{Z(s(t'))}$. In other words, between these two stopping times, the policy $\pi^*_{Z(s(T_i))}$ is taken. 

To continue this pattern, we define a virtual trajectory $\tau'_{T_i} = (s'(t), a'(t))$ defined only for $t\in \{T_i, T_i + 1, ...\}$ as $\tau_{T_i}^{T_{i + 1}}$ concatenated with an instance the realizable trajectory $(s^{virt}(t), a^{virt}(t)) \sim \pi^*_{Z(s(T_i))}\lvert_{s(T_{i + 1})}$. Formally, we define the virtual trajectory $\tau_{T_i}^{T_{i + 1}}$ to consist of 
{\scriptsize
\[s'(t), a'(t) =  \begin{cases} 
      s(t), a(t) & t \in \{T_i, ..., T_{i + 1} -1 \} \\
       s^{virt}(t - T_{i + 1}), a^{virt}(t - T_{i + 1}) & t \in \{T_{i + 1}, ... \}. \\
   \end{cases}
\]
}%
We now have completed the trajectory that takes $\pi^*_{Z(s(T_i))}$ at each step to generate $\tau_{T_i}'$. By construction, it is equivalently distributed to $\pi^*_{Z(s(T_i))}\lvert_{s(T_i)}$ starting at time step $T_i$.

Returning to our proof, we may substitute our virtual trajectory into what we found earlier, to get

{\scriptsize
$\mathbb{E}_{s(T_i) \sim \lambda\lvert_{s}}\bigg[\mathbb{E}_{\tau_{T_i}'} \bigg[ \sum_{t = T_i}^{T_{i + 1} - 1} \gamma^t r(s'(t), a'(t)) \bigg| s(T_i)  \bigg]\bigg]$
}%

which is equivalent to

{\scriptsize
$\mathbb{E}_{s(T_i) \sim \lambda\lvert_{s}}\bigg[\gamma^{T_i}V^{\pi^*_{Z(s(T_i))}}(s(T_i))\bigg] - $\\
$\hspace*{0ex}\mathbb{E}_{s(T_i) \sim \lambda\lvert_{s}}\bigg[\mathbb{E}_{\tau_{T_i}'} \bigg[ \sum_{t = T_{i + 1}}^{\infty} \gamma^t r(s'(t), a'(t)) \bigg| s(T_i) \bigg] \bigg]$.
}%
\\\\
We simplify the second term using similar techniques as above to obtain:\\
{\scriptsize
\begin{align*} 
&\mathbb{E}_{s(T_i) \sim \lambda\lvert_{s}}\bigg[\mathbb{E}_{\tau_{T_i}' } \bigg[ \sum_{t = T_{i + 1}}^{\infty} \gamma^t r(s'(t), a'(t)) \bigg| s(T_i) \bigg] \bigg]\\
& = \mathbb{E}_{\tau_{T_i}^{T_{i + 1}} \sim \lambda\lvert_{s}}\bigg[\mathbb{E}_{\tau_{T_i}' } \bigg[ \sum_{t = T_{i + 1}}^{\infty} \gamma^t r(s'(t), a'(t)) \bigg| \tau_{T_i}^{T_{i + 1}}\bigg] \bigg]\\
& = \mathbb{E}_{\tau_{T_i}^{T_{i + 1}} \sim \lambda\lvert_{s}}\bigg[\mathbb{E}_{\tau_{T_{i + 1}}'} \bigg[ \sum_{t = T_{i + 1}}^{\infty} \gamma^t r(s'(t), a'(t)) \bigg| \tau_{T_i}^{T_{i + 1}}\bigg] \bigg]\\
& = \mathbb{E}_{ s(T_{i + 1})\sim \lambda\lvert_{s}}\bigg[\mathbb{E}_{\tau_{T_{i + 1}}'} \bigg[ \sum_{t = T_{i + 1}}^{\infty} \gamma^t r(s'(t), a'(t)) \bigg|s(T_{i + 1})\bigg] \bigg]\\
& = \mathbb{E}_{ s(T_{i + 1})\sim \lambda\lvert_{s}}\bigg[ \gamma^{T_{i + 1}}V^{\pi^*_{Z(s(T_i))}}(s(T_{i + 1}))\bigg].
\end{align*}
}%
The first equality conditions on additional variables. Similar to before, the second equality eliminates variables that do not appear in the expectation and the third equality evokes the strong Markov property.
\\\\
Substituting this into what we found above, we achieve
{\small
\begin{align*} 
&\mathbb{E}_{\tau \sim \lambda\lvert_{s}} \bigg[ \sum_{t = T_i}^{T_{i + 1} - 1} \gamma^t r(s(t), a(t))\bigg] =\\
&\mathbb{E}_{\tau \sim \lambda\lvert_{s}}\bigg[\gamma^{T_i}V^{\pi^*_{Z(T_i)}}(s(T_i)) 
 - \gamma^{T_{i + 1}}V^{\pi^*_{Z(s(T_i))}}(s(T_{i + 1}))\bigg]
\end{align*}
}%
which is exactly \cref{telescope_condition}.
\end{proof}

\textbf{Proof of \cref{amal_bound}}. With the above preparations, we are now ready to prove  \cref{amal_bound}.

\begin{proof}
For any partition $G$ on $\mathcal N$, let $\pi^*_G(s) = (\pi_g^*(s_g): \forall g \in G)$.
\\\\
Let $T_0 = 0$ and $T_i$ be the time step $t$ for which $Z(s(t)) \neq Z(s(t - 1))$ for the $i$-th time. By definition, for all $t' \in \{T_i, ..., T_{i + 1} - 1\}$ we have $Z(s(T_i)) = Z(s(t'))$ and therefore $\pi^*_{Z(s(T_i))} = \pi^*_{Z(s(t'))}$.

From \cref{amalgam_condition}, we know \cref{telescope_condition} is satisfied and we may use the Telescoping Lemma.

We have \\
$V^{\lambda} (s) \geq \mathbb{E}_{\tau \sim \pi\lvert_{s}}\bigg[V^{\pi^*_{Z(s)}}(s) - \sum_{i = 1}^\infty \gamma^{T_i} \Delta_{i}^\tau\bigg]$
\\\\
and we may bound \\
\begin{align*}
&\Delta_{i}^\tau = V^{\pi^*_{Z(s(T_{i - 1}))}}(s(T_i)) - V^{\pi^*_{Z(s(T_i))}}(s(T_i))\\
&\leq V^{*}(s(T_i)) - V^{\pi^*_{Z(s(T_i))}}(s(T_i))\\
&\leq V^{*}_{Z(s(T_i))}(s(T_i)) - V^{\pi^*_{Z(s(T_i))}}_{Z(s(T_i))}(s(T_i)) + \frac{2}{(1 - \gamma)}\gamma^{c + 1}  \tilde{r}\\
&\leq \frac{2}{(1 - \gamma)}\gamma^{c + 1}  \tilde{r}.
\end{align*}
Where the inequality in the penultimate step is by Lemma \ref{dtl_q} with $\delta = 0$ applied twice. The last step is because $V^{*}_{Z(s(T_i))}(s(T_i)) - V^{\pi^*_{Z(s(T_i))}}_{Z(s(T_i))}(s(T_i)) \leq 0$, using the same reasoning as in \cref{amal_naive}.
\\\\
Plugging in our upper bound for $\Delta_i^\tau$ into our result from the Telescoping Lemma, we get\\
$V^{\lambda} (s) \geq V^{\pi^*_{Z(s)}}(s) - \frac{2}{(1 - \gamma)^2}\gamma^{c + 2}  \tilde{r}$.
\\\\
Therefore, using \cref{amal_naive} we obtain
\begin{align*}
V^*&(s) - V^{\lambda} (s) \\
&\leq V^*(s) - V^{\pi^*_{Z(s)}}(s) + \frac{2}{(1 - \gamma)^2}  \gamma^{c + 2}\tilde{r}\\
& \leq \frac{2}{(1 - \gamma)}\gamma^{c+1}  \tilde{r} + \frac{2}{(1 - \gamma)^2}\gamma^{c + 2}  \tilde{r}\\
&= \frac{2}{(1 - \gamma)^2}\gamma^{c + 1}  \tilde{r}.
\end{align*}
\\
\end{proof}

\subsection{The Cutoff Policy}
\label{cutoff_proof}

Recall the definition of the Cutoff Policy $\chi(s) = (\pi_z^\mathcal C((s_z, \{z\})): \forall z \in Z(s))$. 
This is equivalent to saying $\chi(s) = \pi^\mathcal C ((s,Z(s)))$ by the following equality:

\begin{align*}
\pi^\mathcal C &((s, Z(s))) = \text{argmax}_a Q^*((s, Z(s)), a)\\
& = \text{argmax}_a \sum_{z \in Z(s)} Q^*((s_z, \{z\}), a_z)\\
&= \bigg(\text{argmax}_{a_z}Q^*((s_z, \{z\}), a_z): \forall z \in Z(s)\bigg)\\
&= (\pi_z^\mathcal C((s_z, \{z\})): \forall z \in Z(s)).
\end{align*}

The decomposition in the second equality can be shown similarly to the value decomposition shown in \cref{cutoff}. 

We will also overload notation for this section and say a policy $\pi$ in the Locally Interdependent Multi-Agent MDP may be viewed as a policy in the Cutoff Multi-Agent MDP by letting $\pi((s, C)) = \pi(s)$.

\subsubsection{Step 1)}

We begin by using the Dependence Time Lemma to establish a relationship between policies in the Cutoff Multi-Agent MDP and the Locally Interdependent Multi-Agent MDP.
\begin{lemma}
\label{cutoff_imdp_q}
For any policy $\pi$ in the Locally Interdependent Multi-Agent MDP, we have $\lvert V^\pi(s) - V^\pi((s, Z(s))) \rvert \leq \frac{\gamma^{c + 1}}{1 - \gamma}\tilde r$.\\
\end{lemma}
\begin{proof}
Recall our notation $E^c(P)$ for a partition $P$ are all tuples $(j,k)$ such that agents $j,k$ do not lie in the same partitions.
\\\\
A consequence of the Dependence Time Lemma in the Locally Interdependent Multi-Agent MDP setting is that, for any realizable trajectory $\tau$, with $t \in \{0, ... c\}$ and $\delta_0 \in \{0, ..., t\}$  we have $r(s(t), a(t)) =  \sum_{z \in Z(s(\delta_0))} r_z(s_z(t), a_z(t))$. 
\\\\
With this decomposition, we know that for every $\delta_0$ specified above, we have $\overline r_{j,k}(s_j(t),a_j(t), s_k(t), a_k(t)) = 0$ for $(j,k) \in E^c(Z(s(\delta_0)))$. In otherwords, $\overline r_{j,k}(s_j(t),a_j(t), s_k(t), a_k(t)) = 0$ for $(j,k) \in E^c(\cap_{\delta_0 \leq t}Z(s(\delta_0)))$ using our notation for the intersection of partitions. The quantity $\cap_{\delta_0 \leq t}Z(s(\delta_0))$ is exactly $C^\tau(t)$ by definition, so we may make the decomposition $r(s(t), a(t)) =  \sum_{z \in C^\tau(t)} r_z(s_z(t), a_z(t))$.

Therefore,
{\small
\begin{align*} 
 &V^{\pi} (s) = \mathbb{E}_{\tau \sim \pi\lvert_{s}}\bigg[\sum_{t = 0}^\infty \gamma^{t} \sum_{j, k \in \mathcal N}\overline r_{j,k}(s_j(t), a_j(t), s_k(t), a_k(t))\bigg]\\
 &= \mathbb{E}_{\tau \sim \pi\lvert_{s}}\bigg[\sum_{t = 0}^{c} \gamma^{t} \sum_{z \in C^\tau(t)} r_z(s_z(t), a_z(t)) \\
 & \hspace{15ex}+ \sum_{t = c + 1}^\infty \gamma^{t}\sum_{j, k \in \mathcal N} \overline r_{j, k}(s_{i}(t), a_j(t), s_{k}(t), a_k(t))\bigg]\\
 &= \mathbb{E}_{\tau \sim \pi\lvert_{s}}\bigg[\sum_{t = 0}^\infty \gamma^{t} \sum_{z \in C^\tau(t)} r_z(s_z(t), a_z(t)) \\
 & \hspace{5ex} + \sum_{t = c + 1}^\infty \gamma^{t} \sum_{(j,k) \in E^c(\mathcal C^\tau(t))}\overline r_{j,k}(s_j(t), a_j(t), s_k(t), a_k(t))\bigg]\\
 &:= V^\pi((s, Z(s))) + \xi.
 \\
\end{align*}

}%

This concludes the proof since $\xi \in [-\frac{\gamma^{c + 1}}{1 - \gamma}\tilde r, \frac{\gamma^{c + 1}}{1 - \gamma}\tilde r]$.
\\
\end{proof}

The following is the corresponding lemma to \cref{dtl_q} in the Cutoff Setting. It also can be viewed as the effect the Dependence Time Lemma has on the Q-values for the Cutoff MDP.

\begin{lemma}

\label{dtl_q_cutoff}
Let $(s(t), a(t))$ be any realizable trajectory in the Locally Interdependent Multi-Agent MDP, $\pi$ be any policy in the Cutoff Multi-Agent MDP, and $\delta \in \{0, ..., c\}$. Further, let $T \geq \delta$ is an arbitrary time step. We have $\lvert V^\pi((s(T), Z(s(T)))) - V_{Z(s(T- \delta))}^\pi((s(T), Z(s(T)))) \rvert \leq \frac{\gamma^{c + 1 - \delta}}{1 - \gamma}\tilde r$.\\
\end{lemma}
\begin{proof}

Notice that the realizable trajectory in the statement is in the Locally Interdependent Multi-Agent MDP, and therefore, we may use the Dependence Time Lemma. Similar to what is shown in \cref{dtl_q}, we consider a new trajectory realization $(s'(t'), a'(t'))$ with $s'(0) = s(T)$. For $t' \in \{0, ... c - \delta\}$ we know $\overline r_{j,k}(s_j(t'),a_j(t'), s_k(t'), a_k(t')) = 0$ for $(j,k) \in E^c(Z(s(T - \delta)))$. By the definition of $r^\mathcal C$, we also have $\overline r_{j,k}(s_j(t'),a_j(t'), s_k(t'), a_k(t')) = 0$ for the tuples $(j,k)$ in $E^c(C^\tau(t'))$. Put together, this gives us the following decomposition $r^\mathcal C((s(t'), C(t')), a(t')) 
= \sum_{z \in C^\tau(t') \cap Z(s(T - \delta))} r_z(s_z(t'), a_z(t'))$ on the reward function.

Using this equality, we have\\
{\scriptsize
\begin{align*} 
 &V^{\pi} ((s(T), Z(s(T)))) \\
 &= \mathbb{E}_{\tau \sim \pi\lvert_{s(T)}}\bigg[\sum_{t' = 0}^\infty \gamma^{t'} \sum_{z \in C^\tau(t')} r_{z}(s_z(t'), a_z(t'))\bigg]\\
 &= \mathbb{E}_{\tau \sim \pi\lvert_{s(T)}}\bigg[\sum_{t' = 0}^{c - \delta} \gamma^{t'} \sum_{z \in C^\tau(t') \cap Z(s(T - \delta))} r_z(s_z(t'), a_z(t')) \\
 & \hspace{15ex}+ \sum_{t' = c + 1 - \delta}^\infty \gamma^{t'}\sum_{z \in C^\tau(t')} r_z(s_z(t'), a_z(t'))\bigg]\\
 &= \mathbb{E}_{\tau \sim \pi\lvert_{s(T)}}\bigg[\sum_{t' = 0}^{\infty} \gamma^{t'} \sum_{z \in C^\tau(t') \cap Z(s(T - \delta)) } r_z(s_z(t'), a_z(t')) \\
 & \hspace{0ex} + \sum_{t' = c + 1 - \delta}^\infty \gamma^{t'} \sum_{\substack{(j,k) \in E(C^\tau(t')) \cap \\E^c(Z(s(T- \delta))) }}\overline r_{j,k}(s_j(t'), a_j(t'), s_k(t'), a_k(t'))\bigg]\\
 &:= V_{Z(s(T - \delta))}^\pi((s(T), Z(s(T)))) + \xi.
\end{align*}
}%
The reasoning in most of the steps is similar to \cref{dtl_q}. In the penultimate step, $E(C^\tau(t'))\cap E^c(Z(s(T - \delta)))$ expresses the tuples of agents $(j,k)$ such that $j$, $k$ are in different partitions of $Z(s(T - \delta))$ but are in the same partition in $C^\tau(t')$.

This concludes the proof since $\xi \in [-\frac{\gamma^{c + 1 - \delta}}{1 - \gamma}\tilde r, \frac{\gamma^{c + 1 - \delta}}{1 - \gamma}\tilde r]$.
\\
\end{proof}

\subsubsection{Step 2)}
\textbf{Proof of \cref{cutoff_naive}}. We will next prove the naive policy candidate bound.
\begin{proof}
Using \cref{cutoff_imdp_q}, we can show
\begin{align*}
V^*(s) - &V^{\pi^\mathcal C}((s,Z(s)))\\
&\leq V^*(s) - V^{\pi^*}((s,Z(s)))\\
& \leq \frac{1}{(1 - \gamma)}\gamma^{c+1}  \tilde{r}.
\end{align*}
\\
\end{proof}

\subsubsection{Step 3)}
We say that for partitions $P_1$ and $P_2$ on $\mathcal N$, partition $P_1$ is finer than a partition $P_2$, $P_1 \subset P_2$ if for all $p_1 \in P_1$, there exists $ p_2 \in P_2$ such that $p_1 \subset p_2$. Then, we have the following:
\begin{lemma}
\label{cutoff_condition}
For realizable trajectories in a Locally Interdependent Multi-Agent MDP, define a sequence of stopping times with $T_0 = 0$ and $T_i$ the time step $t$ for which $Z(s(t)) \not\subset Z(s(t - 1))$  for the $i$-th time. Then we have,
 {\scriptsize
 \begin{align*}
&\mathbb{E}_{\tau \sim \chi\lvert_{s}} \bigg[ \sum_{t = T_i}^{T_{i + 1} - 1} \gamma^t r(s(t), a(t))\bigg]= \mathbb{E}_{\tau \sim \chi\lvert_{s}} \bigg[ \gamma^{T_i}V^{\pi^\mathcal C}((s(T_i), Z(s(T_i))))\bigg]\\
  &\hspace*{10ex}-\mathbb{E}_{\tau \sim \chi\lvert_{s}} \bigg[\gamma^{T_{i + 1}} V^{\pi^\mathcal C}((s(T_{i + 1}),  \bigcap_{T_i \leq t' \leq T_{i + 1}}Z(s(t'))))\bigg]
 \end{align*}
 } 
\end{lemma}
\begin{proof}\renewcommand{\qedsymbol}{}
The proof is omitted as it is very similar to \cref{amalgam_condition}.
\end{proof}

\textbf{Proof of \cref{cutoff_bound}}. Finally, with all the previous results in hand, we may prove the main result.
\begin{proof}
Let $T_0 = 0$ and $T_i$ be the time step $t$ for which $Z(s(t)) \not\subset Z(s(t - 1))$  for the $i$-th time. By definition, for all $t' \in \{T_i, ..., T_{i + 1} - 1\}$ the visibility partitions only become finer.

We use \cref{cutoff_condition}, and a minor variation of the Telescoping Lemma (proof very similar to \cref{telescope}) to obtain:\\
$V^{\chi} (s) \geq \mathbb{E}_{\tau \sim \chi\lvert_{s}}\bigg[V^{\pi^\mathcal C}((s, Z(s))) - \sum_{i = 1}^\infty \gamma^{T_i} \Delta_{i}^\tau\bigg]$
\\\\
where we have
{\small
\begin{align*}
&\Delta_{i}^\tau = V^{\pi^\mathcal C}((s(T_i), \bigcap_{T_{i - 1} \leq t' \leq T_{i}}Z(s(t')))) \\
&\hspace{20ex}- V^{\pi^\mathcal C}((s(T_i), Z(s(T_{i}))))\\
& = V^{\pi^\mathcal C}((s(T_i), Z(s(T_{i} - 1) \cap Z(s(T_{i})))) \\
&\hspace{20ex}- V^{\pi^\mathcal C}((s(T_i), Z(s(T_{i}))))\\
& = V^{\pi^\mathcal C}_{Z(s(T_{i} - 1))}((s(T_i), Z(s(T_{i})))) \\
&\hspace{15ex}- V^{\pi^\mathcal C}((s(T_i), Z(s(T_{i}))))\\
&\leq \frac{1}{(1 - \gamma)}\gamma^{c}  \tilde{r},
\end{align*}
}%
where the inequality in the last step is by Lemma \ref{dtl_q_cutoff} with $\delta = 1$. 

Plugging in our upper bound for $\Delta_i^\tau$ into our result from above, we get\\
$V^{\chi} (s) \geq V^{\pi^\mathcal C}((s, Z(s))) - \frac{1}{(1 - \gamma)^2}\gamma^{c + 1}  \tilde{r}$.
\\\\
Therefore, by using \cref{cutoff_naive} we find
\begin{align*}
V^*&(s) - V^{\chi} (s) \\
&\leq V^*(s) - V^{\pi^\mathcal C}((s,Z(s))) + \frac{1}{(1 - \gamma)^2}  \gamma^{c + 1}\tilde{r}\\
& \leq \frac{1}{(1 - \gamma)}\gamma^{c+1}  \tilde{r} + \frac{1}{(1 - \gamma)^2}\gamma^{c + 1}  \tilde{r}\\
&= \frac{2 - \gamma}{(1 - \gamma)^2}\gamma^{c + 1}  \tilde{r},
\end{align*}

and this concludes the proof.
\\
\end{proof}

\subsection{The First Step Finite Horizon Optimal Policy}
Recall the definition of the First Step Finite Horizon Optimal Policy. If $\pi^{*}_{finite} = \{ \pi^{*, 0},  \pi^{*, 1}, ...,  \pi^{*, c}\}$ are the set of discounted finite horizon optimal policies with horizon $c$, then $\phi(s) = \pi^{*, 0}(s)$.

For the purposes of this section, $\tau_t^c \sim \pi^{*}_{finite}\big\lvert _{s}$ will denote a realizable trajectory defined between time steps $t$ and $c$ distributed according to a roll out starting at $s(t) =s$ using the non-stationary policy $\pi^{*, t'}(s(t'))$ for $t' \in \{t, ..., c\}$ to generate actions. $\tau_t^c \sim \pi^{*}_{finite}\big\lvert _{s,a}$ is defined similarly and will condition on $s(t) = s$ and $a(t) = a$.

\subsubsection{Step 1)}
First, we will show that the First Step Finite Horizon Optimal Policy for a Locally Interdependent Multi-Agent MDP can equivalently be computed on the Cutoff Multi-Agent MDP.

Let $V_h^*$, $Q_h^*$ be the values for the discounted finite horizon optimal reward for a Locally Interdependent Multi-Agent MDP with horizon $c$. Similarly $V_h^{\mathcal C}$, $Q_h^{\mathcal C}$ will represent the optimal discounted finite horizon values with horizon $c$ in the Cutoff Multi-Agent MDP. Then we have the following lemma:
\begin{theorem}
    \label{cutoff_finite_proof}
    $Q_0^*(s,a) = Q^{\mathcal C}_0((s, Z(s)), a)$ for any $s,a$.
    
\end{theorem}
\begin{proof}
Recall from the proof of \cref{cutoff_imdp_q}, we have for any realizable trajectory $\tau' = (s'(t'), a'(t'))$ in a Locally Interdependent Multi-Agent MDP, $r(s'(t'), a'(t')) = \sum_{z\in C^{\tau'}(t')} r_z(s_z'(t'), a_z'(t'))$ for $t'\in\{0,...,c\}$.

For $h\in \{0, ..., c\}$ let $\tau_0^{h}$ be any realizable trajectory. In the following, we will append an instance of $\tau_h^c \sim \pi^{\mathcal C}_{finite}\big\lvert _{(s(h), C^\tau(h)),a(h)}$ from $h + 1$ onwards and call this realizable trajectory $\tau$. We find,
{\small
\begin{align*} 
 &Q^{\mathcal C}_h((s(h), C^\tau(h)), a(h)) \\
 &= \mathbb{E}_{\tau_h^c \sim \pi^{\mathcal C}_{finite}\big\lvert _{(s(h), C^\tau(h)),a(h) }}\bigg[\sum_{t = h}^{c} \gamma^{t} \sum_{z \in C^\tau(t)} r_z(s_z(t), a_z(t)) \bigg] \\
 &= \mathbb{E}_{\tau_h^c \sim \pi^{\mathcal C}_{finite}\big\lvert _{(s(h), C^\tau(h)),a(h)}}\bigg[\sum_{t = h}^c \gamma^{t} r(s(t), a(t)))\bigg] \\
 &=  r(s,a) + \\
 &\hspace{0ex}\gamma \mathbb{E} _{s'\sim P(\cdot \lvert s, a)}\bigg[\max_{a'} Q^{\mathcal C}_{h + 1}((s', C^\tau(h)\cap Z(s')), a')\bigg].
\end{align*}
}%

This is nearly identical to the Bellman Equations for the Discounted Finite Horizon Locally Interdependent Multi-Agent MDP problem. 

In fact, we have $Q_h^{\mathcal C} ((s(h), C^\tau(h)), a(h)) = Q_h^{*} (s(h), a(h))$ inductively since $Q_c^{\mathcal C} = Q_c^* = 0$. Therefore, $Q_0^*(s, a) = Q^{\mathcal C}_0((s, Z(s)),a)$.

\end{proof}

Next, we use \cref{cutoff_finite_proof} to prove that the discounted finite horizon optimal policy is a group-decentralized policy.

\begin{corollary}
\label{finite_group_decentralized}
The discounted finite horizon joint optimal policy $\pi^{*, 0}$ for $\mathcal M$ is a group decentralized policy.
\end{corollary}
\begin{proof}
Using \cref{cutoff_finite_proof} and a decomposition similar to the Cutoff Multi-Agent MDP value decomposition shown in \cref{cutoff} we can show that $\pi^{*, 0}$ is group decentralized:
\begin{align*} 
\pi^{*,0} &(s) = \text{argmax}_a Q^*_0(s, a)\\
 & = \text{argmax}_a Q^{\mathcal C}_0((s, Z(s)), a)\\
& = \text{argmax}_a \sum_{z \in Z(s)} Q^{\mathcal C}_0((s_z, \{z\}), a_z)\\
&= \bigg(\text{argmax}_{a_z}Q^{\mathcal C}_0((s_z, \{z\}), a_z): \forall z \in Z(s)\bigg).\\
\end{align*}
\end{proof}

\subsubsection{Step 2)}
\textbf{Proof of \cref{finite_naive}}. We are now ready to prove the bound on the naive policy candidate.
\begin{proof}
Using the definition of the discounted finite horizon optimal policy, we find\\

\begin{align*} 
 V^*(&s) - V_0^{*} (s)\\
&\leq \mathbb{E}_{\tau \sim \pi^*\lvert_{s}} \bigg[\sum_{t = 0}^\infty \gamma^{t} r(s(t), a(t))\bigg] \\
&\hspace{10ex} - \mathbb{E}_{\tau_0^c \sim \pi^{*}_{finite}\big\lvert _{s}}\left[\sum_{t = 0}^{c} \gamma^t r(s(t), a(t)) \right]\\
&\leq \mathbb{E}_{\tau \sim \pi^*\lvert_{s}} \bigg[\sum_{t = c+1}^\infty \gamma^{t} r(s(t), a(t))\bigg] \\
&\leq \frac{1}{(1 - \gamma)}\gamma^{c+1} \tilde{r}.
\end{align*}

The second inequality holds because the optimal discounted finite horizon reward will have a larger expected discounted sum of rewards in the first $c$ iterations. 
\\
\end{proof}

\subsubsection{Step 3)}
\begin{lemma}
\label{finite_condition}
For the First Step Finite Horizon Optimal Policy $\phi$ and and $i\in \{0, 1, ...\}$
\\
$\medmath{\mathbb{E}_{\tau \sim \phi\lvert_{s}}\bigg[\gamma^i r(s(i), a(i))\bigg]\geq}
$\\
$\medmath{\hspace*{5ex}
\mathbb{E}_{\tau \sim \phi\lvert_{s}}\bigg[\gamma^i V_0^*(s(i)) - \gamma^{i + 1} V_0^*(s(i + 1)) \bigg] - \gamma^{i + c + 1} \tilde r} $.
\end{lemma}
\begin{proof}

Recall our notation from \cref{amalgam_condition} for $\tau_T$, $\tau^T$ and $\tau_{T_1}^{T_2}$. It will be used again in this Lemma.

To achieve \cref{telescope_condition} with our specific stopping time and policy, we must bound $\mathbb{E}_{\tau \sim \phi\lvert_{s}} \bigg[ \gamma^i r(s(i), a(i))\bigg] $ for each $i$. 

Similar to \cref{amalgam_condition} we will apply the law of total expectation, eliminate random variables, and apply the Markov property to obtain the equivalence,
{\normalsize
\begin{align*} 
&\mathbb{E}_{\tau \sim \phi\lvert_{s}} \bigg[ \gamma^i r(s(i), a(i))\bigg] \\
& = \mathbb{E}_{\tau^{i} \sim \phi\lvert_{s}}\bigg[\mathbb{E}_{\tau \sim \phi\lvert_{s}} \bigg[  \gamma^i r(s(i), a(i)) \bigg| \tau^{i}\bigg] \bigg]\\
& = \mathbb{E}_{\tau^{i} \sim \phi\lvert_{s}}\bigg[\mathbb{E}_{\tau_{i} \sim \phi\lvert_{s}} \bigg[  \gamma^i r(s(i), a(i)) \bigg| \tau^{i}\bigg] \bigg]\\
& = \mathbb{E}_{s(i) \sim \phi\lvert_{s}}\bigg[\mathbb{E}_{\tau_{i} \sim \phi\lvert_{s}} \bigg[ \gamma^i r(s(i), a(i)) \bigg|s(i) \bigg] \bigg].
\end{align*}
}%

Again as in to \cref{amalgam_condition}, we will construct a virtual trajectory. Here, we will define $(s'(t), a'(t))$ to be $\tau'^{i + c + 1}_i$ for $t \in \{i, ..., i + c + 1\}$  by starting with $s(i), a(i)$ and appending an instance of the trajectory 
$(s^{virt}(t'), a^{virt}(t')) = (\tau^{virt})^c_1 \sim \pi^*_{finite}\lvert_{s(i + 1)}$ starting at time step 1 . Further, we will append what we call a ghost sample of $s^{virt}(c), a^{virt}(c)$. This is a state $s'' \sim P(\cdot \lvert s^{virt}(c),a^{virt}(c))$ and an arbitrary action $a''$.

In other words,
{\scriptsize
\[s'(t), a'(t) =  \begin{cases} 
      s(i), a(i) & t = i  \\
       s^{virt}(t - i), a^{virt}(t - i ) & t \in \{i + 1, ..., i + c \} \\
       s'', a'' & t = i + c + 1 .
   \end{cases}
\]
}%
We use this new realizable trajectory in place of $\tau_i$ by adding and subtracting the rest of the trajectory to obtain,
{\scriptsize
\begin{align*} 
 &\mathbb{E}_{s(i) \sim \phi\lvert_{s}}\bigg[\mathbb{E}_{\tau_{i} \sim \phi\lvert_{s}} \bigg[ \gamma^i r(s(i), a(i)) \bigg|s(i) \bigg] \bigg]\\
 &= \mathbb{E}_{s(i) \sim \phi\lvert_{s}}\bigg[\mathbb{E}_{\tau'^{i + c + 1}_{i}} \bigg[\sum_{t = i}^{i + c + 1} \gamma^t r(s'(t), a'(t)) \bigg|s(i) \bigg] \bigg] \\
 &\hspace{10ex}-  \mathbb{E}_{s(i) \sim \phi\lvert_{s}}\bigg[\mathbb{E}_{\tau'^{i + c + 1}_{i}} \bigg[\sum_{t = i + 1}^{i + c + 1} \gamma^t r(s'(t), a'(t)) \bigg|s(i) \bigg] \bigg].
\end{align*}
}%

We can bound the first quantity by bounding the ghost sample and recognizing that the trajectory leading up to that sample is distributed according to $\pi^*_{finite}\lvert_{s(i)}$ as follows
{\small
\begin{align*} 
 &\mathbb{E}_{s(i) \sim \phi\lvert_{s}}\bigg[\mathbb{E}_{\tau'^{i + c + 1}_{i}} \bigg[\sum_{t = i}^{i + c + 1} \gamma^t r(s'(t), a'(t)) \bigg|s(i) \bigg] \bigg] \\
  &\geq \mathbb{E}_{s(i) \sim \phi\lvert_{s}}\bigg[\mathbb{E}_{\tau'^{i + c}_{i} } \bigg[\sum_{t = i}^{i + c} \gamma^t r(s'(t), a'(t)) \bigg|s(i)\bigg] \bigg]  - \gamma^{i + c + 1} \tilde r \\
  & =  \mathbb{E}_{s(i) \sim \phi\lvert_{s}}\bigg[\gamma^i V_0^*(s(i)) \bigg]  - \gamma^{i + c + 1} \tilde r.
\end{align*}
}%
In the second quantity, notice that the inner expectation is over $\tau'^{i + c + 1}_i$ is a realizable trajectory under a non-stationary policy. Therefore, we may upper bound the discounted sum of rewards under this trajectory for $c$ iterations with the discounted finite horizon optimal expected reward. Using steps similar to before,
{\small
\begin{align*} 
 &\mathbb{E}_{s(i) \sim \phi\lvert_{s}}\bigg[\mathbb{E}_{\tau'^{i + c + 1}_{i}} \bigg[\sum_{t = i + 1}^{i + c + 1} \gamma^t r(s'(t), a'(t)) \bigg|s(i)\bigg] \bigg]\\
 &=\mathbb{E}_{\tau_i^{i + 1} \sim \phi\lvert_{s}}\bigg[\mathbb{E}_{\tau'^{i + c + 1}_{i}} \bigg[\sum_{t = i + 1}^{i + c + 1} \gamma^t r(s'(t), a'(t)) \bigg| \tau_i^{i + 1}\bigg] \bigg]\\
 &=\mathbb{E}_{\tau_i^{i + 1} \sim \phi\lvert_{s}}\bigg[\mathbb{E}_{\tau'^{i + c + 1}_{i + 1}} \bigg[\sum_{t = i + 1}^{i + c + 1} \gamma^t r(s'(t), a'(t)) \bigg| \tau_i^{i + 1}\bigg] \bigg]\\
 & = \mathbb{E}_{s(i + 1) \sim \phi\lvert_{s}}\bigg[\mathbb{E}_{\tau'^{i + c + 1}_{i + 1}} \bigg[\sum_{t = i + 1}^{i + c + 1} \gamma^t r(s'(t), a'(t))\bigg|s(i + 1) \bigg] \bigg]\\
 & \leq  \mathbb{E}_{s(i + 1) \sim \phi\lvert_{s}}\bigg[ \gamma^{i + 1} V_0^*(s(i + 1))\bigg] .
\end{align*}
}%
Substituting these two bounds into our original equation,
{\small
\begin{align*} 
 &\mathbb{E}_{\tau \sim \phi\lvert_{s}}\bigg[\gamma^i r(s(i), a(i))\bigg]\geq  \\
 &\hspace{1ex}\mathbb{E}_{\tau \sim \phi\lvert_{s}}\bigg[\gamma^i V_0^*(s(i)) -  \gamma^{i + 1} V_0^*(s(i + 1)) \bigg] - \gamma^{i + c + 1} \tilde r 
\end{align*}
}%
as desired.

\end{proof}

\textbf{Proof of \cref{finite_bound}}. With all the previous preparations, we are now ready to prove the main result.

\begin{proof}

Let $T_i = i$. Using \cref{finite_condition}, we have
\\
$\medmath{\mathbb{E}_{\tau \sim \phi\lvert_{s}} \bigg[ \gamma^{i} r(s(i), a(i))\bigg] }$\\
\\$\hspace*{10ex}\medmath{\geq
  \mathbb{E}_{\tau \sim \phi\lvert_{s}} \bigg[ \gamma^{i}V^{*}_0(s(i))-\gamma^{i + 1} V^{*}_0(s(i + 1))\bigg] - \gamma^{i + c + 1}\tilde r}$
\\and therefore,\\
$\medmath{\mathbb{E}_{\tau \sim \phi\lvert_{s}} \bigg[ \gamma^{i} (r(s(i), a(i)) + \gamma^{c + 1} \tilde r)\bigg] }$\\
\\$\hspace*{10ex}\medmath{\geq
  \mathbb{E}_{\tau \sim \phi\lvert_{s}} \bigg[ \gamma^{i}V^{*}_0(s(i))-\gamma^{i + 1} V^{*}_0(s(i + 1))\bigg]}$.

Intuitively, we can think about this as \cref{telescope_condition} applied to an MDP with all rewards adjusted by $+ \gamma^{c + 1}\tilde r$.

Using the Telescoping Lemma, we have 
\begin{align*}
V^{\phi} (s) &= (V^{\phi} (s) + \sum_{t = c + 1}^\infty \gamma^t \tilde r) - \frac{1}{1 - \gamma}\gamma^{c+1}\tilde r \\
&\geq \mathbb{E}_{\tau \sim \phi\lvert_{s}}\bigg[V^{*}_0(s) - \sum_{i = 1}^\infty \gamma^{T_i} \Delta_{i}^\tau\bigg] - \frac{1}{1 - \gamma}\gamma^{c+1}\tilde r \\
&= V^{*}_0(s) - \frac{1}{1 - \gamma}\gamma^{c+1}\tilde r,
\end{align*}
where the last equality holds because $\Delta_i^\tau = V^*_0(s(i)) - V^*_0(s(i)) = 0$
\\\\
Therefore, we may substitute our result from the Telescoping Lemma to obtain
\begin{align*}
V^*&(s) - V^{\phi} (s) \\
&\leq V^*(s) - V^{*}_0(s) + \frac{1}{1 - \gamma}\gamma^{c + 1}  \tilde{r}\\
&\leq \frac{2}{1 - \gamma} \gamma^{c + 1} \tilde{r}.
\end{align*}
where \cref{finite_naive} is used in the last inequality.
\\
\end{proof}

\subsection{Lower Bound}

\textbf{Proof of \cref{lower_bound}}. We conclude by providing the construction for the lower bound here.

\begin{figure}[h]
\centering
\begin{tikzpicture}
    \node[shape=circle,draw=black ] (s2) at (-2.75,1) {$S_1$};
    \node[shape=circle,draw=black ] (s4) at (-1.25,1) {$S_2$};
    \node[shape=circle,draw=black, text opacity=0] (s5) at (-2,2) {$S_3$};
    \node[shape=circle,draw=black, text opacity=0 ] (s6) at (-1,3) {$S_3$};
    \node[shape=circle,draw=black ] (s7) at (0,4) {$S_5$};
    \node[shape=circle,draw=black , text opacity=0 ] (s8) at (1,3) {$S_3$};
    \node[shape=circle,draw=black , text opacity=0 ] (s9) at (2,2) {$S_3$};
    \node[shape=circle,draw=black ] (s10) at (2.75,1) {$S_{3}$};
    \node[shape=circle,draw=black ] (s12) at (1.25,1) {$S_{4}$};
    \node[shape=circle,draw=black ] (s13) at (0,5.5) {$S_{6}$};
    
    \path [->](s2) edge node[left] {} (s5);
    \path [->](s4) edge node[left] {} (s2);
    \path [](s5) -- node[auto=false, sloped] {\ldots} (s6);
    \path [->](s6) edge node[left] {} (s7);
    \path [->](s12) edge node[left] {} (s9);
    \path [->](s10) edge node[right, yshift= 1 ex] {$a_0$} (s9);
    \path [->](s10) edge node[below] {$a_1$} (s12);
    \path [](s9) -- node[auto=false, sloped] {\ldots} (s8);
    \path [->](s8) edge node[left] {} (s7);
    \path [->](s7) edge node[left] {} (s13);
    \path [->](s13) edge node[left] {} (s7);

    \draw [decorate,decoration={brace,amplitude=5pt,raise=4ex}]
  (-2, 2) -- (-1,3) node[midway,yshift=+3em, sloped]{$\ell$};
    \draw [decorate,decoration={brace,amplitude=5pt,raise=4ex}]
  (1, 3) -- (2,2) node[midway,yshift=+3em, sloped]{$\ell$};
\end{tikzpicture}
\caption{Lower Bound Construction $\mathcal M_\ell$}
\label{lower_figure}
\end{figure}
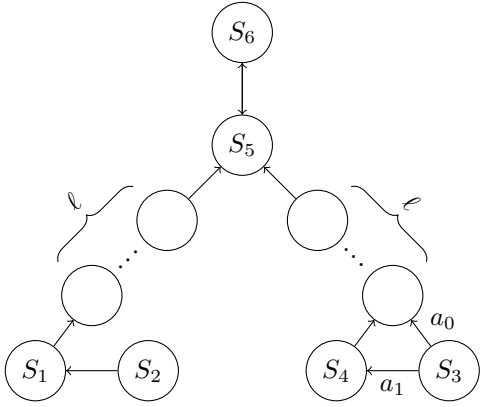

\begin{proof}
The construction is a completely deterministic MDP with 2 agents and the only non-trivial action can be taken at $S_3$ shown in \cref{lower_figure}. The MDP consists of no rewards and will only suffer penalty $-\tilde r$ if agents overlap ($\mathcal R = 0$). Notice that states $S_5$, $S_6$ transition to each other deterministically.

Here, $c = \ell$ because $\mathcal V = 2 \ell + 1$ implies $ c = \lfloor \frac{\mathcal V - \mathcal R}{2}\rfloor = \lfloor \ell + \frac{1}{2}\rfloor = \ell$. When agents are at state $S_1, S_3$, the optimal action is for the $S_3$ agent to take $a_1$ so they do not collide. Similarly, when agents are at $S_2, S_3$, the optimal is for the agent at $S_3$ to take action $a_0$.

Denote the trivial action as $X$ and $p_0 = P(\pi(S_3) = a_0)$. 
Notice,
$$\lvert V((S_1, S_3)) \rvert = \lvert p_0 Q((S_1, S_3), (X, a_0)) \rvert $$
$$\hspace{10ex}= p_0 \frac{\gamma^{\ell + 1}}{1 - \gamma} \tilde r = p_0 \frac{\gamma^{c + 1}}{1 - \gamma} \tilde r \geq  p_0 \frac{\gamma^{c + 2}}{1 - \gamma} \tilde r $$
and
$$ \lvert V((S_2, S_3)) \rvert = \lvert (1- p_0) Q((S_2, S_3), (X, a_1)) \rvert $$
$$\hspace{10ex}= (1 - p_0) \frac{\gamma^{\ell + 2}}{1 - \gamma} \tilde r = (1 - p_0) \frac{\gamma^{c + 2}}{1 - \gamma} \tilde r. $$
Since $V^*((S_1, S_3)) = V^*((S_2, S_3)) = 0$, 
\\
either $\lvert V^*((S_1, S_3)) - V((S_1, S_3)) \rvert \geq \frac{1}{2} \frac{\gamma^{c + 2}}{1 - \gamma} \tilde r$ or $ \lvert V^*((S_2, S_3)) - V((S_2, S_3))\rvert \geq \frac{1}{2} \frac{\gamma^{c + 2}}{1 - \gamma} \tilde r$ must be true.
\\
\end{proof}

\end{document}